\pdfoutput=1 
\documentclass[11pt]{article}
\usepackage{amsmath,amssymb,amsthm,mathtools}
\usepackage[margin=1in]{geometry}
\usepackage{graphicx}
\usepackage[hidelinks]{hyperref}
\hypersetup{pdftitle={Poor Man's Agentic Modeling: Simulating Large LLM-Agent Societies on a Laptop},
  pdfauthor={Igor Itkin}}
\usepackage{microtype}
\usepackage{algorithm}
\usepackage{algpseudocode}
\newtheorem{proposition}{Proposition}
\newtheorem{lemma}{Lemma}
\theoremstyle{definition}

\newtheorem{consequence}{Consequence}
\theoremstyle{remark}

\newcommand{\R}{\mathbb{R}}

\newcommand{\E}{\mathbb{E}}
\newcommand{\Var}{\operatorname{Var}}
\DeclareMathOperator{\tanhh}{tanh}

\title{\bf Poor Man's Agentic Modeling:\\
Simulating Large LLM-Agent Societies on a Laptop}
\author{Igor Itkin\\
\small Independent Researcher, Tel Aviv, Israel\\
\small \texttt{ig.itkin@gmail.com}\\
\small ORCID: \href{https://orcid.org/0009-0004-9513-8463}{0009-0004-9513-8463}}
\date{July 2026}

\begin{document}
\maketitle

\begin{abstract}
Simulating societies of many large language model (LLM) agents is expensive, yet the questions
asked of such simulations are usually macroscopic: phase behaviour, stylised facts, and scaling
with the number of agents $N$, not the cognition of any single agent. We turn a statistical-physics
observation into a method: replace each LLM agent by a low-parameter model fitted from a few hundred to a few thousand cheap
queries, then run the society at any $N$ on a laptop. Whether this works is decided before the
simulation runs, chiefly by what each agent perceives. We introduce an
[interaction order $\times$ memory] taxonomy that maps perception and memory to an effective theory
and a predicted $N$-trend of the surrogate error. We validate it on a faithful reimplementation of
the LLM macroeconomy EconAgent and seven further named LLM simulations, with agent decisions cloned
from genuine LLM elicitations (primarily DeepSeek) for a few dollars; the predicted error trends
hold cell by cell, and the two refuted predictions, both on a strongly saturating response and traced
to its curvature, are themselves matched quantitatively by the theory with no free parameters.
\end{abstract}

\medskip
\noindent\textbf{Keywords:} LLM multi-agent systems; agent-based modelling; mean-field theory;
finite-size scaling; surrogate models; behavioural cloning; recommender systems; sociophysics.

\section{Introduction}\label{sec:intro}

A growing body of work builds simulations in which each agent is a large language model: generative
towns, LLM macroeconomies, social-media societies, and epidemic models with reasoning citizens.
These simulations are valuable because they reproduce human-like macroscopic behaviour (business
cycles, opinion polarisation, epidemic waves) without hand-coded behavioural rules. They are also
expensive. A single run of a thousand-agent society can cost tens of dollars and tens of hours in
API calls, which places systematic study, and in particular the study of how behaviour scales with
the number of agents $N$, out of reach for most researchers.

The expense buys per-agent cognition, yet the scientific questions posed to these simulations are
almost always macroscopic: does a phase transition occur, what are the stylised facts of the
aggregate, how does an observable scale with $N$? For such questions a century of statistical
physics offers a lesson: the macroscopic behaviour of a large interacting system is governed by a
few collective variables, and most microscopic detail is irrelevant to it. If that stance applies
to LLM societies, then for macroscopic purposes each expensive agent can be replaced by a cheap
surrogate whose few parameters are fitted from a small number of queries, and the society can be
studied at any $N$ on a laptop.

Whether the stance applies is not automatic, and the contribution of this paper is a criterion for
when it does. We argue that the deciding property is \emph{what each agent perceives}. An agent that
reacts to one population-wide aggregate (an inflation rate, a global trending feed) sits in a
mean-field regime, and a scalar surrogate reproduces the macroscopic observable with an error that
vanishes as $N^{-1/2}$. An agent that reacts to a signal shared only within its community, or only
to its graph neighbours, sits in a regime where the same scalar surrogate carries an error that
does not vanish, and may even grow with $N$. The taxonomy of Section~\ref{sec:framework} makes this
precise by mapping the perception and memory design of a simulation to a cell, an effective theory,
and a predicted $N$-trend of the surrogate error. Because the perception design of a social
simulation is in practice set by its recommender, the recommender is the switch between the regimes
in which cheap modelling succeeds and those in which it fails.

We validate the criterion in three ways. First, on a faithful, code-authoritative reimplementation
of the LLM macroeconomy EconAgent, we show that the surrogate reproduces the target's macroscopic
signatures and that doing so exposes what those signatures do and do not measure. Second, we
falsify the taxonomy directly: its cell assignments are pre-registered and then tested blind, both
on held-out contact graphs and on the measured response function of a real LLM. Third, we classify
and reproduce eight named LLM simulations spanning all three perception cells: EconAgent, AgentTorch,
OASIS, AgentSociety, De Marzo et al.'s consensus game, Williams et al.'s generative
epidemic, LLMTraveler's congestion game, and Generative Agents' Smallville (with TwinMarket's financial stylised facts as a
documented boundary case), together with a cross-domain check against a differentiable agent-based model, using
agent responses cloned throughout from genuine LLM decisions (DeepSeek by default; six models for the
De Marzo cross-model test, and GPT-4o alongside DeepSeek for TwinMarket).

Two findings of independent interest emerge. EconAgent's frequently cited reproduction of Okun's
law turns out to be an accounting identity that a behaviour-free policy already satisfies, whereas
its Phillips curve is a genuine behavioural signature carried by a single labour-cyclicality
coefficient; we estimate that coefficient from cloned decisions and recover the macroscopic value
as an out-of-sample prediction. And when the pipeline is driven by genuine LLM decisions, a
$2\times2$ ablation isolates the reasoning step---not the wording of the prompt---as the cause of
the emergent Phillips curve, so the cheap surrogate becomes an instrument: what makes the macro law
appear is itself measurable.

\section{Related work}\label{sec:related}

We connect several literatures that are rarely joined; a fuller map of all 115 verified sources is given in the companion review~\cite{review}. \emph{Expensive LLM simulations} such as Generative Agents~\cite{genagents}, EconAgent~\cite{econagent}, OASIS~\cite{oasis}, and
AgentSociety~\cite{agentsociety} establish that LLM societies reproduce human-like macroscopic phenomena and define the
observables a surrogate must hit, but none is paired with a low-parameter model whose scaling is
analysed.

\emph{Sociophysics} supplies the off-the-shelf few-parameter rules (voter and Ising
models, bounded-confidence dynamics, kinetic opinion equations, and active-matter flocking dynamics~\cite{tonertu}), leaving open which rule a given LLM
agent realises. \emph{Coarse-graining and equation learning}, from mean-field reductions to
equation-free methods and archetype-based population models, compress dynamics but stop short of the
finite-$N$ closure problem and of LLM-specific targets. \emph{Mean-field game theory}~\cite{mfg} and
\emph{heterogeneous-agent macroeconomics}~\cite{krusell} provide the $N\to\infty$ limit and three cautions we
adopt: one moment is often enough but misses the tails, small-$N$ exponents can agree by accident,
and a shared driver can manufacture a power law without a phase transition. Finally, \emph{closure
theory} (the Mori--Zwanzig memory kernel~\cite{moriz}, the BBGKY hierarchy~\cite{bbgky}, graphons~\cite{graphon}, and inverse-Ising
methods) supplies the machinery that our taxonomy invokes cell by cell.

The closest prior work is MF-LLM~\cite{mfllm}, which couples a population-level mean field to per-agent LLM decisions; it
keeps the LLM in the loop and analyses no scaling, whereas we replace the agent and study the
macroscopic observable as $N$ varies. The contribution of this paper is the join that no prior work
makes: a low-parameter surrogate per agent, fitted from genuine LLM decisions elicited on a named
LLM simulation's own prompts to reproduce its macroscopic observable, together with a classificatory
layer that predicts the
$N$-trend of the surrogate error and is validated blind.

\section{The perception-ordered taxonomy}\label{sec:framework}

\begin{figure}[t]
  \centering
  \includegraphics[width=\linewidth]{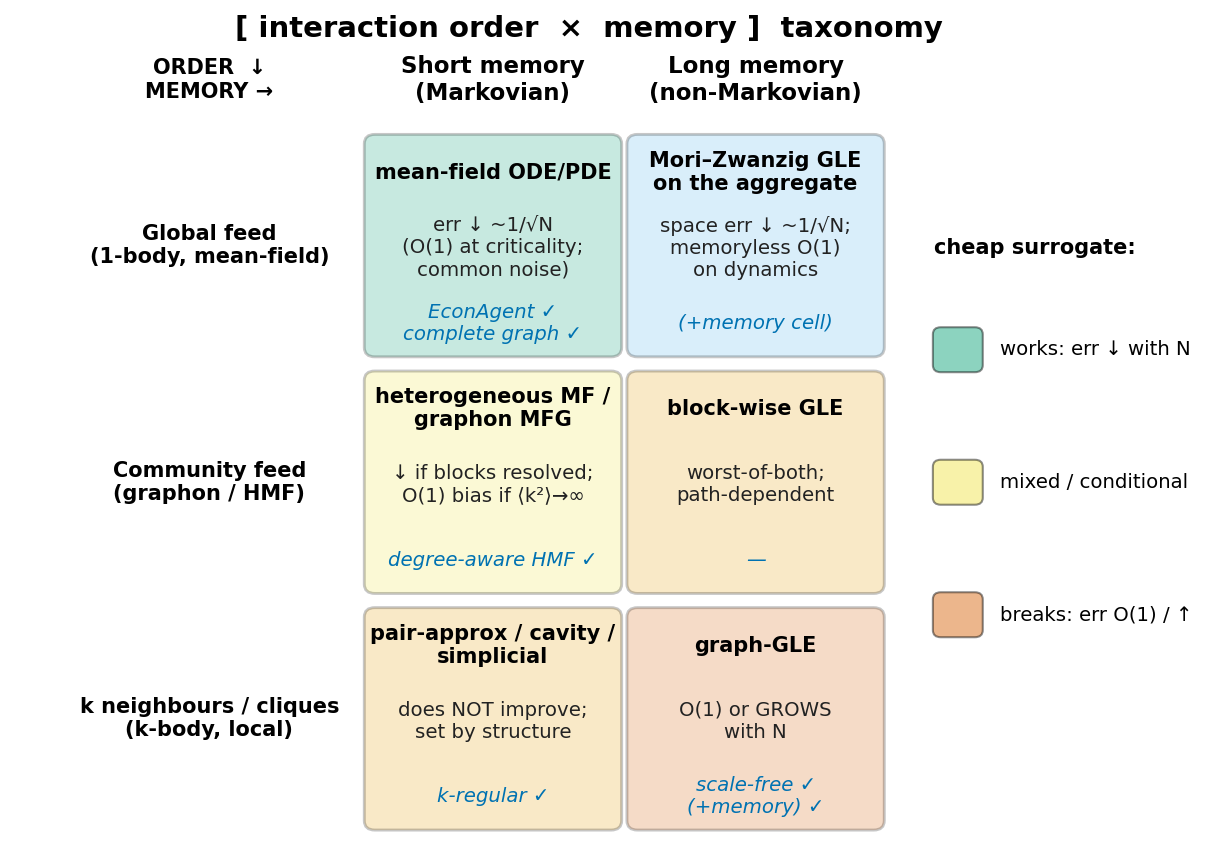}
  \caption{\textbf{The perception cell sets the scaling of the surrogate error.} A simulation's
  perception and memory design places it in one cell of the [interaction order $\times$ memory]
  taxonomy. Each cell lists three things: its effective theory (top, bold), the predicted trend of
  the scalar-surrogate error with $N$ (middle), and the system(s) that validate that cell here
  (bottom, italic). The tint marks whether the cheap surrogate works (error falls with $N$), is
  conditional, or breaks (error stays $O(1)$ or grows), as keyed at the right.}
  \label{fig:taxonomy}
\end{figure}

Let $\Phi_t$ denote the microscopic update of the LLM society and $P$ the projection onto the
macroscopic observable of interest. A cheap surrogate replaces $\Phi_t$ by a low-parameter map
$\hat\Phi_t$ acting on the projected variables. The surrogate reproduces the observable exactly
when coarse-graining commutes with the dynamics, $P\Phi_t = \hat\Phi_t P$; in general it does not,
and the size of the commutation defect $\lVert P\Phi_t - \hat\Phi_t P\rVert$ is what the taxonomy
predicts. We call the resulting gap in the macroscopic observable the \emph{surrogate error}; the
commutation defect is the one-step, map-level quantity that drives it, and the two scale together
(they coincide at the mean-field cell).

In plain terms, the surrogate replaces every agent by a single average agent, so it can be right
only to the extent that the population is well summarised by its average, and what decides that is
what each agent looks at. When every agent reacts to the same population-wide quantity (an inflation
rate, a global trending feed), the agents move together, the mean is a sufficient statistic, and the
only error is sampling noise that washes out as $1/\sqrt N$. When each agent instead reacts to a
private draw or to its own neighbourhood, the agents no longer share an input; averaging discards the
structure that actually drives them, and the error stops falling with $N$: it plateaus, or even
grows. Perception thus decides whether averaging is legitimate at all, and memory whether the past must
be tracked.

We organise this defect along two coarse-graining axes: interaction order and memory.

The first is the \emph{interaction order}: how many other agents feed into a single agent's decision. A
global aggregate feed is order zero (every agent sees the same population statistic) and yields a
mean-field theory in which the scalar surrogate error is set by sampling noise and vanishes as
$N^{-1/2}$. A community feed, shared within each of a fixed number of blocks, is a heterogeneous
mean field: the error no longer vanishes but falls only with the number of blocks, leaving an
$O(1)$ floor in $N$. A local feed, restricted to graph neighbours, is genuinely $k$-body, and the
error is controlled by the degree structure rather than by $N$. The second is \emph{memory}:
whether an agent's decision depends only on current inputs or on an accumulated internal state. Long
memory makes the dynamics non-Markovian and requires a memory kernel in the closure.

These two axes define the cells of Figure~\ref{fig:taxonomy}, and each cell names both a predicted
$N$-trend of the surrogate error and the minimal closure that removes it: a scalar mean field for
the global cell, a block or graphon mean field for the community cell, and a \emph{pair
approximation} (a moment closure that tracks two-agent correlations rather than only single-agent
means~\cite{pairapprox}) with a memory kernel for the local, long-memory cell.

Two further axes refine the picture and are demanded by data presented later. A shared driver that itself fluctuates makes the mean field
random, so its fluctuations fall more slowly than the naive rate. And a strongly curved per-agent
response makes coarse-graining fail via Jensen's inequality even under a private feed, because
the average of a nonlinear response is not the response at the average. Collecting these,
$\lVert P\Phi_t - \hat\Phi_t P\rVert$ is bounded, heuristically, by a sum of an interaction-order
term, a memory term, and a response-curvature term. The first two are the classical BBGKY and
Mori--Zwanzig contributions, and the third is the curvature axis that Section~\ref{sec:perception}
isolates on a real LLM. We treat this decomposition as an organising heuristic rather than a
theorem: the terms are not derived and the constant is not bounded.

\subsection{An exactly solvable case}\label{sec:theory}

For a tractable class of agents, all three terms in the commutation bound
$\lVert P\Phi_t-\hat\Phi_t P\rVert$ are exact rather than heuristic, and working out the perception
term yields a quantitative prediction we confirm later on a real LLM (Section~\ref{sec:perception}). Let $A(g)=\E_x f(x,g)$
be the infinite-population aggregate response to a perceived signal $g$, and suppose the population
is partitioned into $B$ equal communities, community $b$ perceiving $g^\ast+\delta_b$ with
$\delta_b$ independent and $\mathcal N(0,\sigma^2)$. Write $D=A(g^\ast+\delta)-A(g^\ast)$ for a single
community, with mean $m_1=\E D$ and variance $v_1=\Var D$. Since $\E\delta=0$, Jensen's
inequality~\cite{jensen1906} makes $m_1=\E[A(g^\ast+\delta)]-A(g^\ast)$ nonzero for small $\sigma$
whenever $A''(g^\ast)\neq0$ (to leading order $m_1=\tfrac12 A''(g^\ast)\sigma^2$); we call it the
\emph{Jensen bias}.
The scalar mean-field surrogate predicts $A(g^\ast)$, and the realised $B$-community aggregate is
$\tfrac1B\sum_b A(g^\ast+\delta_b)=A(g^\ast)+\tfrac1B\sum_b D_b$, so its error is
$\bigl\lvert\tfrac1B\sum_b D_b\bigr\rvert$.

\begin{proposition}[Community floor]\label{prop:floor}
If $A$ is affine, then
\[
  \mathrm{Floor}(B)=\E\Bigl\lvert\tfrac1B\textstyle\sum_b D_b\Bigr\rvert
  = \E\bigl\lvert\mathcal N(m_1,\,v_1/B)\bigr\rvert;
\]
that is, the scalar mean-field surrogate error floor equals exactly the mean of a \emph{folded normal}
distribution~\cite{leone1961}. More generally, for nonlinear $A$ the same expression gives the
leading-order asymptotic approximation to the floor.
\end{proposition}

\begin{consequence}[Symmetric versus curved regimes]\label{cons:floor}
If the response is symmetric about the operating point, then $D$ is odd in the mean-zero perturbation
$\delta$, so $m_1=\E D=0$ and $\mathrm{Floor}(B)=\sqrt{2v_1/(\pi B)}$ decays as $B^{-1/2}$. If the
response is curved, then $m_1\neq0$, the decay stalls at $\mathrm{Floor}(\infty)=\lvert m_1\rvert$,
and the $B^{-1/2}$ law breaks.
\end{consequence}

A private feed is the limit $B=N$: every agent draws its own misperception, so the same expression
with $B=N$ gives the mean-field error as a function of population size. It decreases only until the
sampling spread $\sqrt{v_1/N}$ falls to the Jensen bias.

\begin{proposition}[The knee $N^\ast$]\label{prop:knee}
Under a private feed the scalar surrogate error equals $\E\lvert\mathcal N(m_1,v_1/N)\rvert$. It
decreases as $N^{-1/2}$ up to the knee
\[
  N^\ast=\frac{v_1}{m_1^{2}}\;\approx\;\frac{4\,A'(g^\ast)^2}{A''(g^\ast)^2\,\sigma^2},
\]
and plateaus beyond it at the curvature floor $\lvert m_1\rvert\approx\tfrac12\lvert A''(g^\ast)\rvert\sigma^2$.
\end{proposition}

\begin{consequence}[Finite versus infinite knee]\label{cons:knee}
A curved response therefore has a finite knee, computable from the fitted response before any
simulation is run; a near-linear response has $N^\ast\to\infty$ and improves as $N^{-1/2}$
indefinitely.
\end{consequence}

The remaining term (interaction order) and the memory term are exact in the opposite sense: they
are \emph{zero} at the mean-field cell. Take a conditionally linear society
$x_{i,t+1}=a\,x_{i,t}+f(g_t)+\xi_{i,t}$ with $\xi_{i,t}$ independent, mean zero, variance $\sigma^2$,
where $g_t=\frac1N\sum_j x_{j,t}$ is the global mean perceived identically by every agent, and let
$P$ project onto the population mean.

\begin{proposition}[Exact commutation at the mean-field cell]\label{prop:commute}
For \emph{any} response $f$, $\E[g_{t+1}\mid g_t]=a\,g_t+f(g_t)=:\hat\Phi(g_t)$. The commutation
defect $\lVert P\Phi_t-\hat\Phi_t P\rVert$ therefore vanishes as $N\to\infty$, and the finite-$N$
defect equals the $O(N^{-1/2})$ fluctuation of the mean noise $\frac1N\sum_i\xi_{i,t}$.
\end{proposition}

\begin{consequence}[Mean-field exactness and curvature irrelevance]\label{cons:commute}
In conditional expectation the mean-field surrogate is therefore exact, and the curvature of $f$
contributes nothing: every agent perceives the identical $g_t$ (contrast the heterogeneous
perception of Proposition~\ref{prop:floor}).
\end{consequence}

Propositions~\ref{prop:floor}--\ref{prop:commute} are proved in Appendices~\ref{app:floor}--\ref{app:commute};
Proposition~\ref{prop:commute} makes the global-feed defect
$C/\sqrt N$ identically for an affine and for a curved response, while adding a heterogeneous
misperception reinstates the $O(1)$ Jensen floor of Proposition~\ref{prop:floor}; both are borne out
on the real LLM in Section~\ref{sec:perception}. The first two propositions are confirmed in
Section~\ref{sec:perception}: the near-linear consumption head has $N^\ast\approx2.5\times10^4$ and
keeps improving, while the saturating work head has $N^\ast\approx29$ and its private-feed error is
already flat at the predicted floor $0.018$ across $N=100$ to $3200$. Together the three propositions
make the schematic bound exact at the mean-field cell and along its first step off each axis: the
defect is $O(N^{-1/2})$ under a global feed, acquires the curvature floor of
Propositions~\ref{prop:floor}--\ref{prop:knee} under heterogeneous perception, and would acquire the
memory and interaction-order terms under long memory and a local feed respectively.

\subsection{From one-step defect to observable error}\label{sec:propagation}

Propositions~\ref{prop:floor}--\ref{prop:commute} bound the \emph{one-step closure defect}
$\varepsilon_t:=\lVert P\Phi_t-\hat\Phi_t P\rVert$ (the commutation defect of
Section~\ref{sec:framework}), whereas Section~\ref{sec:results} measures the error
of a macroscopic observable read off a whole surrogate trajectory. These are three distinct quantities,
and a standard propagation bound fixes when the first controls the last. Write $\hat z_t$ for the
surrogate trajectory ($\hat z_{t+1}=\hat\Phi_t(\hat z_t)$, $\hat z_0=Px_0$), $e_t:=\lVert Px_t-\hat z_t\rVert$
for the \emph{trajectory error}, and $O$ for an $L_O$-Lipschitz observable, so that the
\emph{observable error} equals $\lvert O(Px_T)-O(\hat z_T)\rvert$ and is at most $L_O\,e_T$.

\begin{lemma}[Error propagation]\label{lem:propagate}
Suppose the surrogate map $\hat\Phi_t$ is $L$-Lipschitz for every $t$ and the one-step defect satisfies
$\varepsilon_t\le\varepsilon$ for every $t<T$. Then the trajectory error obeys
\[
  e_T\ \le\ \sum_{t=0}^{T-1}L^{\,T-1-t}\,\varepsilon_t\ \le\ \varepsilon\,\frac{L^{T}-1}{L-1},
\]
and the observable error is at most $L_O\,e_T$. In particular, a contractive surrogate ($L<1$) has
$e_T\le\varepsilon/(1-L)$ uniformly in the horizon. A neutral one ($L=1$) has $e_T\le T\varepsilon$. An
expanding one ($L>1$) may amplify the defect geometrically, so the one-step floor controls the observable
only over horizons $T\lesssim 1/\log L$.
\end{lemma}

The bound is the discrete Gr\"onwall recursion $e_{t+1}\le\varepsilon_t+L\,e_t$ (the triangle inequality
plus $L$-Lipschitz continuity of $\hat\Phi_t$), unrolled from $e_0=0$.

\begin{consequence}[When the one-step floor is what the observable sees]\label{cons:propagate}
Reading a macroscopic observable's accuracy off the one-step floor and knee of
Propositions~\ref{prop:floor}--\ref{prop:knee} is legitimate whenever the coarse dynamics are
non-expanding: there the observable errors of Section~\ref{sec:results} inherit the same $N$-trend as the
one-step defect. They can depart from it only where the surrogate map is locally expanding, which is why
the near-critical cells of Figure~\ref{fig:taxonomy} are the ones on which the trend must be read over a
longer run rather than a single step.
\end{consequence}

\section{Methods}\label{sec:methods}

The method is one procedure, applied unchanged to every target in Section~\ref{sec:results}
(Algorithm~\ref{alg:recipe}): classify the simulation's perception cell, which \emph{predicts} the
$N$-trend of the surrogate error before any fitting; screen the observable; elicit a few hundred to a
few thousand genuine LLM decisions on the target's own prompts; clone a low-parameter surrogate;
read off its error floor and knee from the fitted response; and run the surrogate society to large
$N$ on a laptop, validating the macroscopic observable and checking the error trend against the
cell's prediction. The remainder of this section states the problem formally and details each
component; the parenthetical step numbers refer to Algorithm~\ref{alg:recipe}.

\begin{algorithm}[t]
\caption{The poor-man's recipe: replace each LLM agent by a low-parameter surrogate, then scale the society.}
\label{alg:recipe}
\begin{algorithmic}[1]
\Require a target LLM simulation $S$ with its perception and memory specification; a macroscopic observable $M$; the published value $M^\star$ at conditions $C$
\Ensure a low-parameter surrogate reproducing $M$ at any size $N$ where the cell permits, and a prediction (made from the cell, before any fitting) of how the surrogate error scales with $N$
\State \textbf{Classify} the perception cell: $(\text{order},\text{memory})\gets\Call{classify\_cell}{S}$ \Comment{predicts the error trend before any fitting}
\State \textbf{Screen} $M$: if a behaviour-free policy already reproduces it, reject $M$ \Comment{it is then an accounting identity}
\State \textbf{Elicit} decisions: $\mathcal{D}\gets$ query the LLM on $S$'s own prompts over a small state grid \Comment{cached; a few dollars}
\State \textbf{Clone} the surrogate: $\theta\gets\Call{clone}{\mathcal{D}}$ \Comment{fit the $2$--$12$ parameters by behavioural cloning}
\State \textbf{Read} the floors: $(\text{floor},N^\ast)\gets\Call{perception\_floors}{\theta}$ \Comment{error floor and knee $N^\ast$ from the fitted response}
\State \textbf{Sweep} $N$ geometrically with $N\gg N_{\text{target}}$: simulate the surrogate society and measure $M(N)$
\State \textbf{Validate}: compare $M(N)$ to $M^\star$ at $C$ and test the error trend against the cell's prediction \Comment{pre-registered}
\end{algorithmic}
\end{algorithm}

\textbf{Problem statement.} We are given a target LLM society with microscopic update $\Phi_t$, a
macroscopic observable fixed by the projection $P$ of Section~\ref{sec:framework}, its published
value $M^\star$ at conditions $C$, and a target size $N$. We seek a low-parameter surrogate map
$\hat\Phi_{t}$ with parameter vector $\theta\in\R^d$, $d$ small, whose per-agent policy
$\hat\pi_\theta$ is fitted from real LLM decisions, such that coarse-graining commutes with the
surrogate dynamics: the objective is to minimise the finite-size surrogate error
$\varepsilon(N)=\lVert P\Phi_t-\hat\Phi_t P\rVert$ read out through the observable. Crucially, the
parameters are fitted by behavioural cloning of the per-agent policy,
$\theta=\arg\min_\theta\sum_{(s,a)}\ell\bigl(a,\hat\pi_\theta(s)\bigr)$ over a transfer set of
teacher decisions $(s,a)$, \emph{not} by matching $M^\star$; the macroscopic observable is therefore
never fitted and is always an out-of-sample test, and the taxonomy of
Section~\ref{sec:framework} predicts the trend of $\varepsilon(N)$ from the design's cell alone.

The primary target that instantiates this problem is EconAgent~\cite{econagent} (step~6's
environment), a macroeconomy in which each household is a GPT agent that decides monthly whether to
work and how much to consume; the released code reproduces business-cycle signatures including a
Phillips and an Okun relation. We reimplement its market mechanics in NumPy so that a surrogate can
be dropped into an identical environment, and the reimplementation is code-authoritative: it chains
the four published components (labour, progressive taxation, consumption, and savings) and
reproduces the sign and magnitude of the released non-LLM baseline.

Before an observable is used as a validation target it must pass an a-priori screen (step~2),
because not every macroscopic relation measures behaviour: we ask whether a behaviour-free policy
already produces it, and if so the observable is an accounting identity that validates nothing. This
discriminator is fixed in advance and applied uniformly.

The behavioural content that survives is
captured by cloning (steps~3--4). A surrogate agent is a twelve-parameter student: two logistic
heads, one for the work decision and one for the consumption propensity, each linear in six features
(a bias, three standardised state variables, the interest rate, and the macroscopic signal
$g$), fitted from a transfer set of (state, action) decisions produced by a teacher. The teacher may
be a hand-specified policy or a real LLM, and a single loader makes the two kinds of trace
interchangeable, so the same pipeline serves the controls and the LLM experiments. For the LLM
experiments the teacher is DeepSeek (\texttt{deepseek-chat}, queried through its
Anthropic-compatible endpoint), prompted on EconAgent's own monthly household prompts; decisions are
cached to disk, so a fitted result re-runs at no cost, every paid run carries a hard budget guard,
and the total DeepSeek spend across the whole study is a few dollars.

The classification and finite-size steps (1 and~5) are carried by a small reusable module:
\texttt{classify\_cell} maps a perception and memory spec to a cell, and \texttt{perception\_floors}
returns the mean-field error floors of a fitted response. We validate these primitives against
systems whose scaling is known exactly (the Minority Game and network epidemic models) before
trusting them on LLM targets.

\section{Results}\label{sec:results}

The experiments answer three questions. \textbf{RQ1 (prediction):} does a simulation's perception
cell predict the sign of the surrogate error's $N$-trend (vanishing, plateauing, or
growing) \emph{before} any agent is cloned? \textbf{RQ2 (reproduction):} can a two- to
twelve-parameter surrogate, fitted from a few dollars of genuine LLM decisions, reproduce the
published macroscopic behaviour of named LLM simulations across all three cells: the number within
noise where matching it does not require the target's own model, and the mechanism and functional
form where it does? \textbf{RQ3 (mechanism):} for a given macroscopic observable, is it behavioural at all, and
which single microscopic ingredient carries it? RQ1 is tested by the blind, pre-registered graph and
perception experiments (Sections~\ref{sec:blind}--\ref{sec:perception}) and by the finite-size trend
in every cell; RQ2 by EconAgent (Section~\ref{sec:econagent}) and the remaining named targets
(Appendix~\ref{app:suite}); RQ3 by the a-priori observable test and the
identifiability and De Marzo analyses (Sections~\ref{sec:econagent}, \ref{sec:ident},
\ref{sec:demarzo}).

\subsection{EconAgent in the mean-field cell}\label{sec:econagent}

EconAgent's published targets are a Phillips correlation of $-0.619$ and an Okun correlation of
$-0.918$. The a-priori test separates them. In EconAgent real GDP is, by construction, an
affine-invertible function of the number of working agents (an $R^2=0.996$ fit), so Okun's law
relates a quantity to an affine image of itself; a behaviour-free $\mathrm{Bernoulli}(0.5)$ work
policy already yields an Okun correlation of $-0.998$. Reproducing Okun therefore validates nothing.

The Phillips curve is different: wage inflation is driven by goods-market imbalance with no direct
employment-to-wage channel, so a negative unemployment--inflation relation is a genuine behavioural
signature. Its mechanism is a single procyclical-labour coupling: work propensity rising with the
price signal. Fitting the twelve-parameter student to a procyclical teacher recovers that coupling,
and the macroscopic Phillips correlation, which never enters the fit, emerges as an out-of-sample
prediction at $-0.569 \pm 0.138$, within noise of both the teacher and the published target
($\lvert\Delta\rvert = 0.05$). A non-procyclical teacher cloned through the identical pipeline gives
a vanishing Phillips; the coupling, not the pipeline, produces the effect. This is
the mean-field cell in action: as $N$ grows from $20$ to $500$ the cloned Phillips strengthens
monotonically from $-0.29$ to $-0.86$ as averaging over more agents reinforces the weak signal.

\subsection{Identifiability is a frontier, not a degeneracy}\label{sec:ident}

One macroscopic number need not pin the microscopic mechanism. We map the reachable Phillips
frontier of three candidate micro-channels (Table~\ref{tab:frontier}). Only labour keyed to the
price signal reaches the published value, and near that frontier the channel is nearly pinned,
though the pinning requires the coupling to sit at the upper end of its estimated range. Degeneracy
is local rather than global: it occurs only at an intermediate value of the Phillips correlation
where a labour and a consumption channel coexist, and there a single additional aggregate
observable (the sign of consumption propensity against unemployment, or of employment against the
price level) separates them. So one macroscopic number is more identifying than a naive
under-determination concern would suggest, and where it is not, one more number suffices.

\begin{table}[t]
\centering
\begin{tabular}{lcc}
\hline
Micro channel & most-negative Phillips reachable & reaches $-0.62$?\\
\hline
price-keyed labour ($\mathrm{proc}\cdot g$) & $-0.80$ & yes\\
countercyclical consumption & $\approx -0.43$ & no (saturates)\\
employment-feedback labour & $+0.40$ (wrong sign) & no\\
\hline
\end{tabular}
\caption{\textbf{The reachable Phillips frontier of three micro-channels.} Only price-keyed labour
reaches the published $-0.62$; near the frontier the mechanism is nearly identified by the single
macroscopic number.}
\label{tab:frontier}
\end{table}

\subsection{The triple join on genuine LLM decisions}\label{sec:triple}

We now drive the pipeline with real LLM decisions: fit the student from a budget of DeepSeek
decisions elicited on EconAgent's household prompts, drop it into the market, and recover the
macroscopic observable. The runs cost \$0.44 and \$0.67 for three and six thousand decisions.

The headline is that \emph{asking the model to reason} is what produces the Phillips curve, and a
$2\times2$ ablation isolates it. EconAgent's protocol gives GPT a reasoning channel through a
quarterly reflection; the ablation below is our own elicitation probe. A naive ``reasoning'' prompt
confounds three changes (a chain-of-thought step, an amplified inflation wording, and an intensity
adjective), so we cross whether the model reasons with whether the inflation signal is amplified,
holding the numeric input to the student byte-identical across all four cells (Table~\ref{tab:2x2}).

The two factors are orthogonal: \emph{reasoning} turns the chain-of-thought step on or off, while
\emph{wording} presents the same inflation figure in plain or in amplified language. Reading the four
cells, with no reasoning the cloned Phillips is weak and even flips sign with wording ($-0.43$ plain,
$+0.04$ amplified), whereas with reasoning it is strongly negative under both wordings ($-0.73$ and
$-0.66$): turning reasoning on (moving down a column) shifts the correlation sharply, and once
reasoning is on the wording barely moves it ($-0.73$ vs $-0.66$); wording matters only when reasoning
is off, where it flips the sign. The reasoning effect is large and clearly resolved even where it is
smallest: at plain wording, turning reasoning on shifts the correlation by $0.30$ at a difference
standard error near $0.04$ (the main effect across both wordings is $0.50$), whereas the two
reasoning cells are only marginally separated, so we do not claim one is robustly more negative than
the other.

\begin{table}[t]
\centering
\begin{tabular}{lcc}
\hline
clone Phillips at $N=100$ (chat, 24 seeds) & plain wording & amplified wording\\
\hline
no reasoning & $-0.43 \pm 0.18$ & $+0.04 \pm 0.21$\\
reasoning (CoT) & $-0.73 \pm 0.10$ & $-0.66 \pm 0.12$\\
\hline
\end{tabular}
\caption{\textbf{Reasoning, not wording, drives the emergent Phillips curve.} A $2\times2$ ablation;
the numeric signal fed to the student is identical across the four cells.}
\label{tab:2x2}
\end{table}

The reproduction is model- and reasoning-conditional rather than robust. Under a reasoning prompt
DeepSeek-chat's cloned Phillips is $-0.665 \pm 0.12$, the nearest cell landing $\lvert\Delta\rvert
\approx 0.05$ from the published $-0.619$; but sibling models under the same prompt scatter to
$-0.78$ and $-0.84$, and no configuration reproduces the published value robustly. The exact value
is a model-and-prompt fingerprint. The cheap surrogate is thus both a cost-saving device and an
instrument: what makes the macroscopic law shift (here, whether the agent reasons) is itself a
measurement. The clone also independently reproduces the mean-field scaling of
Section~\ref{sec:econagent}, its terse-prompt Phillips strengthening from $-0.34$ at $N=20$ to
$-0.61$ at $N=500$, so all three legs of the join close on genuine LLM data.

\subsection{Closure-machinery checks on known ground truth}\label{sec:controls}

Before trusting the taxonomy on LLM targets we certify its closure machinery on systems whose
scaling is known exactly. On epidemic dynamics over contact graphs, a scalar mean field is accurate
on a complete graph, degrades on a $k$-regular graph, and fails near threshold on a scale-free graph~\cite{psv}
and on a real Facebook network, exactly as the interaction-order axis predicts. On the Minority Game~\cite{minority}, a finite-size-scaling data collapse recovers the critical control parameter to within
$13\%$; we therefore treat roughly $15\%$ as the toolkit's resolution floor and do not read
precision below it. These controls are epidemics and games, not LLM agents; their role is to
validate the machinery, not the LLM claims.

\subsection{A blind test of the taxonomy on held-out graphs}\label{sec:blind}

To make this a test rather than a fit, we pre-registered each cell's assignment and its predicted
$N$-trend, then evaluated them on contact graphs held out from all calibration.

All three predictions held.

The mean-field cell's error shrank with $N$, the community cell's settled to an $O(1)$ floor,
and the local cell's tracked the degree structure; at large $N$ the three errors ordered themselves
exactly as the assignment demanded. This validates the interaction-order axis on graph dynamics. The
same axis, on a real LLM's response function, is tested next (Section~\ref{sec:perception}).

\subsection{A blind test on the LLM perception layer}\label{sec:perception}

The sharpest test of the perception axis is on a real LLM's response function rather than on a
synthetic one. We fit the student to a genuine DeepSeek reasoning trace and then feed a population
three signals with the same mean and, for the private and community cases, the same misperception
variance, differing only in correlation structure: a global feed that every agent shares, a private
feed drawn independently per agent, and a community feed shared within each of a fixed number of
blocks. The prediction, pre-registered, is that private noise averages away while community noise
leaves a floor that falls only with the number of blocks.

On the near-linear consumption head all five pre-registered predictions held: the global and private
errors fall as $N^{-1/2}$, the community error is flat in $N$ at an $O(1)$ floor, that floor falls
as $B^{-1/2}$ in the number of communities, and a block-aware surrogate repairs it. The strongly
saturating work head refutes two of its five predictions, and we report the refutation as measured:
its private-feed error does not shrink but sits at a floor of $0.018$.

A post-hoc analysis, not
pre-registered, identifies the cause as response curvature: the infinite-population Jensen bias
$\lvert\E_\varepsilon f(g^\ast+\varepsilon) - f(g^\ast)\rvert$ equals $0.018$, matching the floor.
The same curvature constant then predicts, with no further fitting, why the community floor's
$B^{-1/2}$ decay is broken on this head: the block-averaged floor cannot fall below the Jensen
level, so its ratio across $B=5$ to $80$ is compressed to about $2.0$ against the ideal $4.0$ and
the measured $1.93$, while the odd stance response of Appendix~\ref{sec:targets} has zero Jensen bias by
symmetry and obeys the clean law.

The curvature term of the commutation heuristic thus turns from a
post-hoc diagnosis into a confirmed quantitative prediction (Figure~\ref{fig:perception}). It also
confirms Proposition~\ref{prop:knee}: the work head's fitted response gives a knee
$N^\ast=v_1/m_1^2\approx29$, so its private-feed error should already be flat at the floor by
$N=100$, as observed across $N=100$ to $3200$; the consumption head gives $N^\ast\approx2.5\times
10^4$, so it keeps improving throughout, as observed.

\begin{figure}[t]
  \centering
  \includegraphics[width=0.98\linewidth]{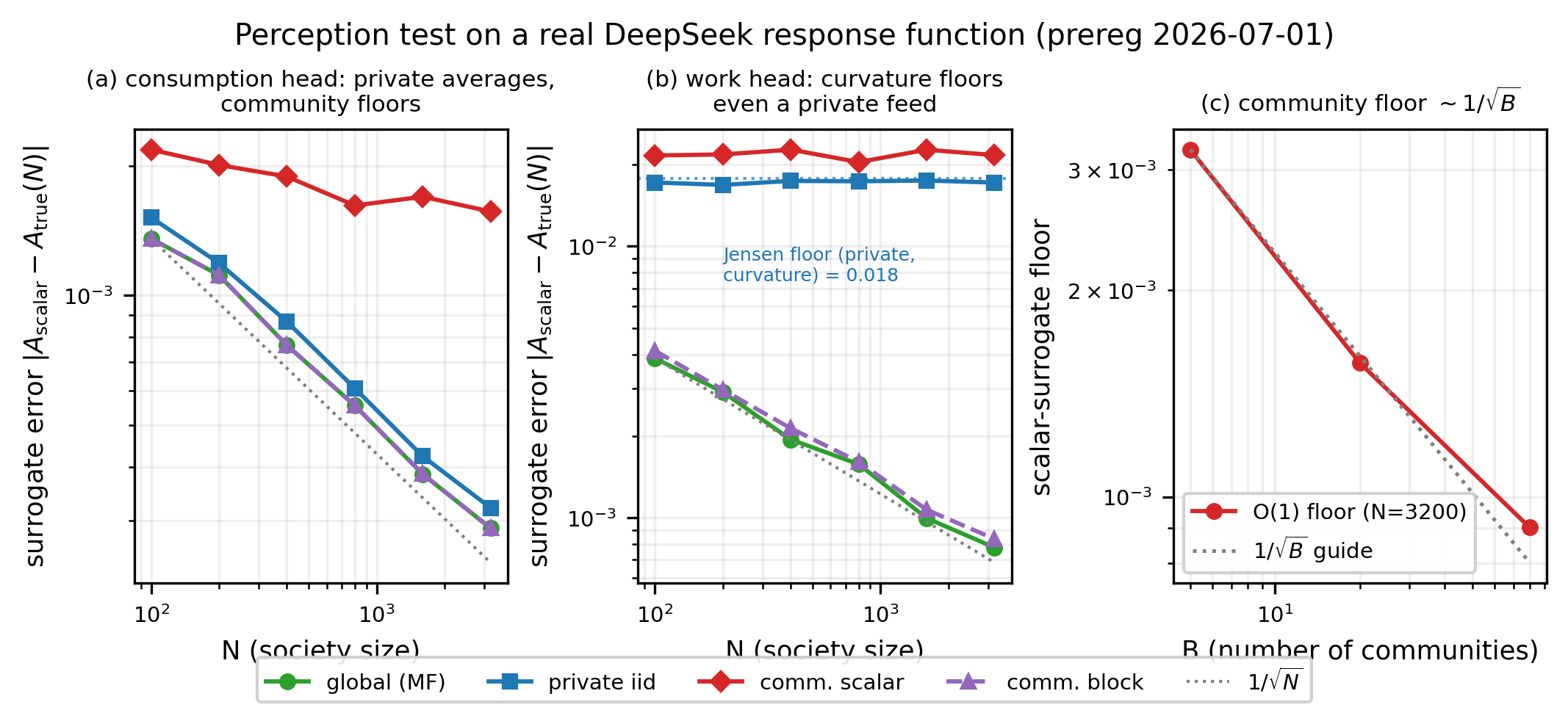}
  \caption{\textbf{The perception switch on a real DeepSeek response function.} Private
  misperception averages away; correlated community misperception leaves an $O(1)$ floor that falls
  as $B^{-1/2}$; a strongly curved response leaves a Jensen floor even under a private feed.}
  \label{fig:perception}
\end{figure}

Both mechanisms live on a single axis. Holding the total misperception variance fixed and letting a
knob $\lambda$ set the fraction that is community-shared rather than private sweeps a real DeepSeek
society across the mean-field boundary (Figure~\ref{fig:dial}): the near-linear head sweeps cleanly
from a vanishing error at $\lambda=0$ to an $O(1)$ floor at $\lambda=1$, the saturating head floors
at the Jensen level for all $\lambda$, and the block-aware surrogate leaves a residual equal to the
Jensen floor scaled by the private fraction. Because the variance is fixed, this isolates
correlation as the driver.

\begin{figure}[t]
  \centering
  \includegraphics[width=0.98\linewidth]{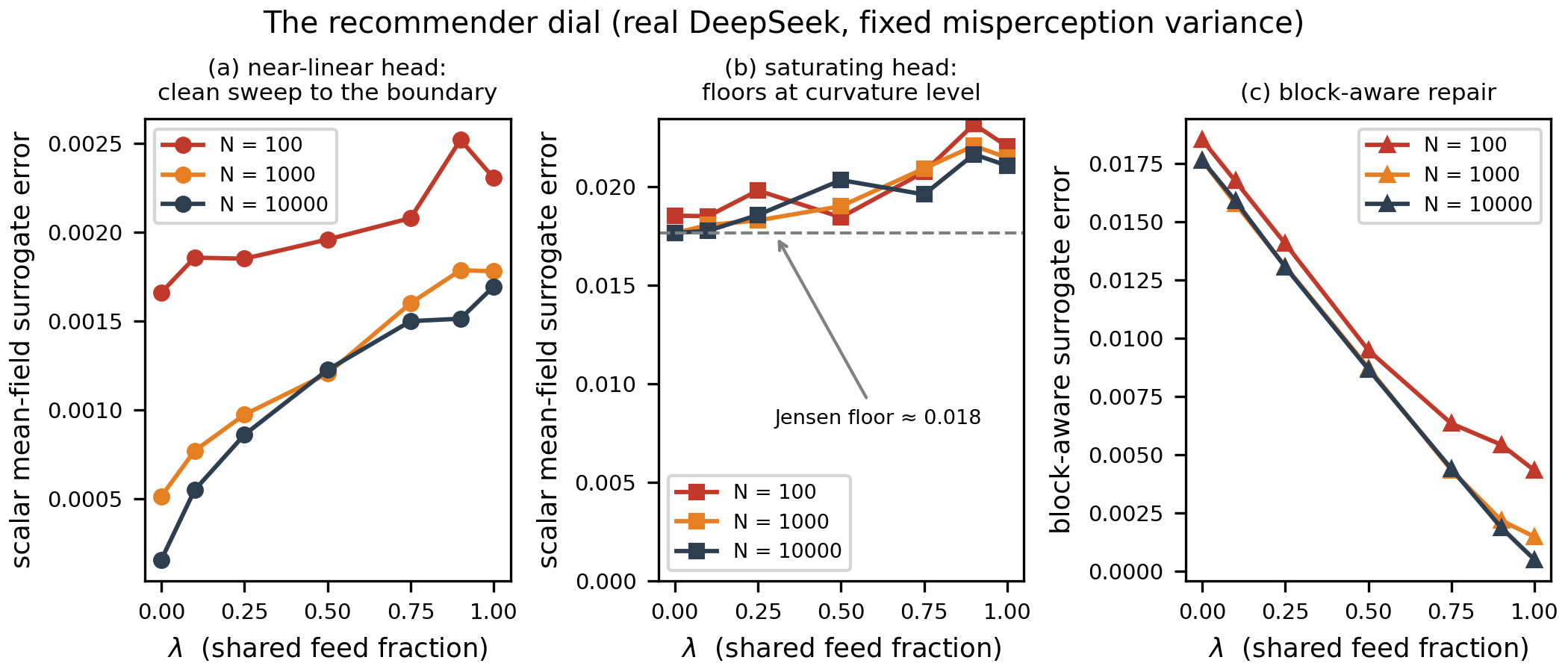}
  \caption{\textbf{The recommender dial.} One knob $\lambda$ (the community-shared fraction of the
  feed, at fixed misperception variance) moves the society across the predictability boundary; the
  two mechanisms (correlation break and curvature floor) separate on the one axis.}
  \label{fig:dial}
\end{figure}

To separate what the taxonomy fixes from what the individual model fixes, we refit $f_{\rm LLM}$ on
thirteen elicitation traces spanning five independent labs (DeepSeek, OpenAI, Anthropic, Google,
and Meta) and both plain and reasoning prompts, and re-ran the whole perception test on each
(\texttt{run\_xmodel\_perception.py}, pre-registered before the traces were collected). The three
predictions that define coarse-graining---the global feed averages out, the correlated community
feed leaves an $O(1)$ floor, and a block-aware closure repairs it---held for twelve of the thirteen.
The single exception is \texttt{gpt-4o-mini}, the smallest model in the panel: its block-aware
closure moves in the predicted direction, with the block error falling monotonically relative to the
scalar error, but does not reach the factor-of-two criterion by $N=3200$. Every other model,
including the three OpenAI flagships \texttt{gpt-4o}, \texttt{gpt-4.1}, and \texttt{gpt-5.1}, clears
the core, so the failure tracks model capability rather than any one lab's style. The one prediction
that flips across models is that a \emph{private} feed also averages out, and it flips with each
model's response curvature: the Jensen bias runs from $\approx10^{-4}$ (private error falls about
fivefold with $N$) to $0.018$ (private error flat in $N$), and the knee $N^\ast$ moves from
$\sim\!10^{5}$ down to $\sim\!140$ in step, as Proposition~\ref{prop:knee}
predicts.\footnote{These cross-model knees are read empirically from the private-feed scan; for the
primary DeepSeek trace this gives $\sim\!140$, whereas the analytic $v_1/m_1^2$ of
Section~\ref{sec:perception} gives $\approx29$ for the same trace. The two agree to an order of
magnitude, the gap reflecting the folded-normal approximation in the empirical estimator.} The cell
assignment is therefore a property shared across architectures rather than a house style of one
model family; the response curvature that sets the knee is a property of the individual model, and
the two vary independently.

\subsection{De Marzo et al.: the critical group size is a perception limit}\label{sec:demarzo}

De Marzo, Castellano and Garc\'ia~\cite{demarzo} give a published universal result: an LLM shown its
peers' opinions adopts the majority with probability $P(m)=\tfrac12[\tanhh(\beta m)+1]$, governed by
one majority-force $\beta$, with \emph{consensus} (a macroscopic state in which the population
aligns on one opinion) requiring $\beta>1$ and a critical group size $N_c$ where
$\beta(N)=1$. They report $N_c$ growing with model capability: $N_c\approx50$ for Llama-3-70B, and
only lower bounds ($\gtrsim1000$) for their more capable models, which they could not push to
consensus failure within the tested range. Across models $N_c$ correlates strongly with the MMLU capability
benchmark ($r\approx0.75$). This target lets us do more than
reproduce a mechanism: we can ask \emph{why} a finite critical group size exists at all, and the
answer turns out to be the paper's own variable---perception.

We fit $\beta(N)$ by maximum likelihood from real decisions elicited with their verbatim prompt
across six models (Figure~\ref{fig:demarzo}, Table~\ref{tab:demarzo}). The measurement reproduces
their universal form, and it exposes the mechanism. A finite $N_c$ cannot come from the mean-field
response: the self-consistency $m=\tanhh(\beta m)$ has a nonzero (consensus) solution for every $N$
when $\beta>1$, and finite-$N$ fluctuations only round the transition. Formally (the proof is in Appendix~\ref{app:demarzo}):

\begin{proposition}[Consensus threshold and its crossing]\label{prop:demarzo}
The iterated response $m\mapsto\tanhh(\beta m)$ is odd with Jacobian $\beta$ at $m=0$. Its disordered
fixed point $m=0$ is therefore stable iff $\beta<1$ and loses stability at $\beta_c=1$ through a
supercritical pitchfork~\cite{strogatz} (oddness excludes a quadratic term), with ordered branch
$m^\ast\simeq\sqrt{3(\beta-1)}$ as $\beta\downarrow1$. Hence consensus exists iff $\beta>1$,
independently of $N$.
\end{proposition}

\begin{consequence}[Finite $N_c$ is a perception threshold]\label{cons:demarzo}
A finite critical group size therefore exists \emph{if and only if} the measured slope $\beta_{\rm
eff}(N)$ decays through $1$. The slope $\beta_{\rm eff}(N)$ is the resolution with which an agent reads
a weak majority in a list of $N$ opinions, so $N_c$ is a perception threshold, not a thermodynamic one.
\end{consequence}

The data confirm both halves. A strong majority ($m=0.5$) is read perfectly ($P=1$) at every $N$;
only the resolution of a \emph{weak} majority degrades with $N$, and $\beta_{\rm eff}(N)$ \emph{is}
that resolution. Models split cleanly by whether their resolution decays (Table~\ref{tab:demarzo}):
the reasoning models Opus-4.8 and GLM-5.2 count perfectly ($\beta_{\rm eff}$ pegged, $N_c=\infty$),
DeepSeek holds a flat $\beta_{\rm eff}\approx1.9>1$ ($N_c=\infty$, our pre-registered lower bound),
while GPT-4o, Llama and GPT-4-Turbo show a decaying $\beta_{\rm eff}$ and hence a finite $N_c$. The
level is a red herring: GPT-4o's $\beta_{\rm eff}\approx7$ far exceeds DeepSeek's $1.9$ yet GPT-4o
loses consensus first, because only the \emph{asymptote} relative to $1$ matters.

This aligns with
the Weber/approximate-number-system law~\cite{weber}: numerosity discrimination depends on the
\emph{ratio} $n_k/n_z=(1+m)/(1-m)$, which is $N$-independent at fixed majority fraction, so a perfect
ratio-perceiver would have flat $\beta_{\rm eff}$ and $N_c=\infty$; a finite $N_c$ is a deviation from
that ideal (attention mass $1/N$ diluted over a long list). The literature supports the substrate (LLM magnitude representations are log-compressive, Weber-like~\cite{weber}), though it cautions that
this representational geometry does not by itself guarantee ideal behavioural ratio-perception, so we
lean on it as motivation rather than proof. De Marzo et al.\ themselves attribute the $\beta$-decline to an
information-processing limit, and report $N_c$ correlating with MMLU ($r\approx0.75$)---capability,
not thermodynamics.

\begin{table}[t]
\centering
\begin{tabular}{lccc}
\hline
model & $\beta_{\rm eff}(N)$ trend & our $N_c$ & De Marzo \\
\hline
Opus-4.8 & pegged (perfect count) & $\infty$ & --- \\
GLM-5.2 & pegged & $\infty$ & --- \\
DeepSeek-chat & flat $\approx1.9$ & $\infty$ & --- \\
GPT-4-Turbo & slow decay & $\sim1600$ & $\gtrsim1000$ (l.b.) \\
GPT-4o & decay & $\sim900$ & unreported \\
Llama-3.3-70B & decay & $\sim1600$ & $\approx50$ (Llama-3) \\
\hline
\end{tabular}
\caption{\textbf{The critical group size is set by whether perception resolution decays.} Models
whose $\beta_{\rm eff}(N)$ stays above $1$ have $N_c=\infty$; those whose resolution decays through
$1$ have a finite $N_c$, recovered from the finite-$N$ naming-game dynamics on the fitted response.
The load-bearing result is the \emph{finite/infinite split}, which is consistent with De Marzo et
al.\ wherever the models overlap: the one model they pin as finite (Llama-3) we also find finite, and
the models we find flat post-date their study. Absolute values we do
\emph{not} claim to match: our static probe over-estimates $\beta_{\rm eff}$ at small $N$, we used
Llama-3.3 rather than their Llama-3, and De Marzo et al.\ pin only Llama-3-70B ($\approx50$) with the rest
lower bounds, so the fine ordering is not established either way. GPT-4-Turbo's $\sim1600$ is at
least consistent with their $\gtrsim1000$ lower bound.}
\label{tab:demarzo}
\end{table}

We probe the mechanism interventionally. Handing GPT-4o the explicit tally ($n_k$, $n_z$) in the
prompt (removing the list-reading load while leaving the social decision unchanged) \emph{flattens}
its $\beta_{\rm eff}(N)$ from a decay to a pegged constant, sending $N_c\to\infty$; DeepSeek, already
flat, is unchanged. The decay that generates a finite $N_c$ is therefore the list-to-count perception
step, not the opinion dynamics. (The intervention also lengthens the prompt by one sentence, a
length/salience change we did not separately control; a pure-perception control that asks only which
opinion is \emph{more common} (no adoption) corroborates, giving GPT-4o a flat $\beta$ in $N$, so
the effect is not a generic prompt-length artefact.) A consequence for this paper's own agenda: $N_c$ is \emph{not} a
finite-size-scaling critical point; it is a bifurcation in an externally driven perception control
parameter, with no diverging correlation length, so a data-collapse in $N$ is the wrong lens here;
one should measure $\beta_{\rm eff}(N)$ and locate its crossing of $1$.

\begin{figure}[t]
  \centering
  \includegraphics[width=0.9\linewidth]{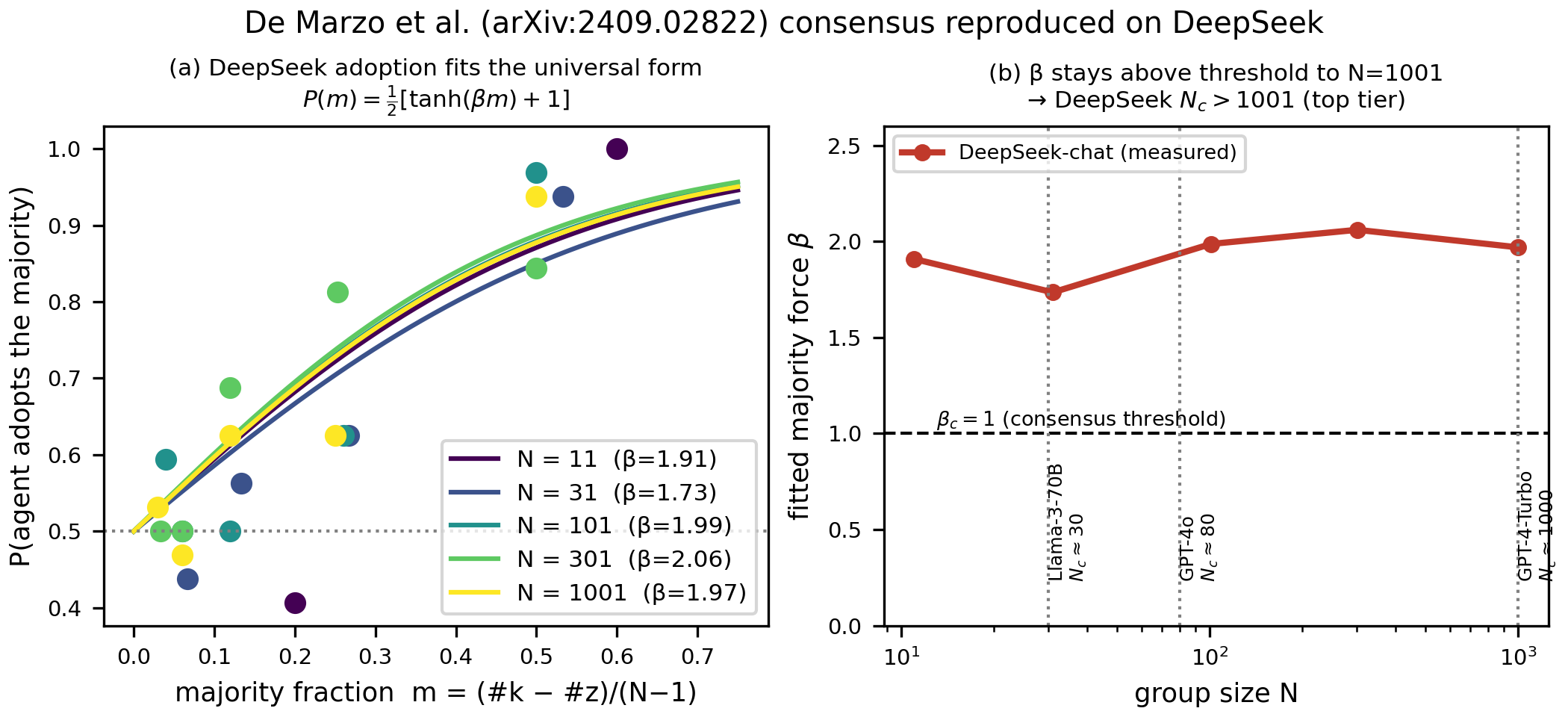}
  \caption{\textbf{De Marzo consensus is a perception crossing.} The published adoption form transfers
  to a new model (left); the majority force $\beta_{\rm eff}(N)$ is the perception resolution, and a
  finite critical group size exists only where it decays through the threshold $\beta_c=1$ (right).}
  \label{fig:demarzo}
\end{figure}

\subsection{A measured memory kernel}\label{sec:memory}

Appendix~\ref{sec:targets}'s honest weakness is that only the assimilation rate is fitted while
the memory that drives polarisation is posited. We measure it. Eliciting DeepSeek attitude updates
given a controlled history of past interactions, we recover a discrete Mori--Zwanzig kernel by
regression, $a_{\text{next}}-a \approx \sum_\tau K(\tau)\,(p_{t-\tau}-a)$ (Figure~\ref{fig:memory}).
All three pre-registered predictions held. The current-interaction weight $K(0)=0.27$ is in the
range of the independently fitted $\mu=0.415$ and, tellingly, below it, since an existing history
damps the current move: the conviction-braking signature seen directly. The past-interaction tail
is real, at $47\%$ of $K(0)$, so the update is genuinely non-Markovian. And the tail is concentrated
in the first few lags and vanishes beyond, so a finite memory closure captures it. The memory axis is
thus grounded empirically in the LLM layer, alongside the perception axis.

\begin{figure}[t]
  \centering
  \includegraphics[width=0.7\linewidth]{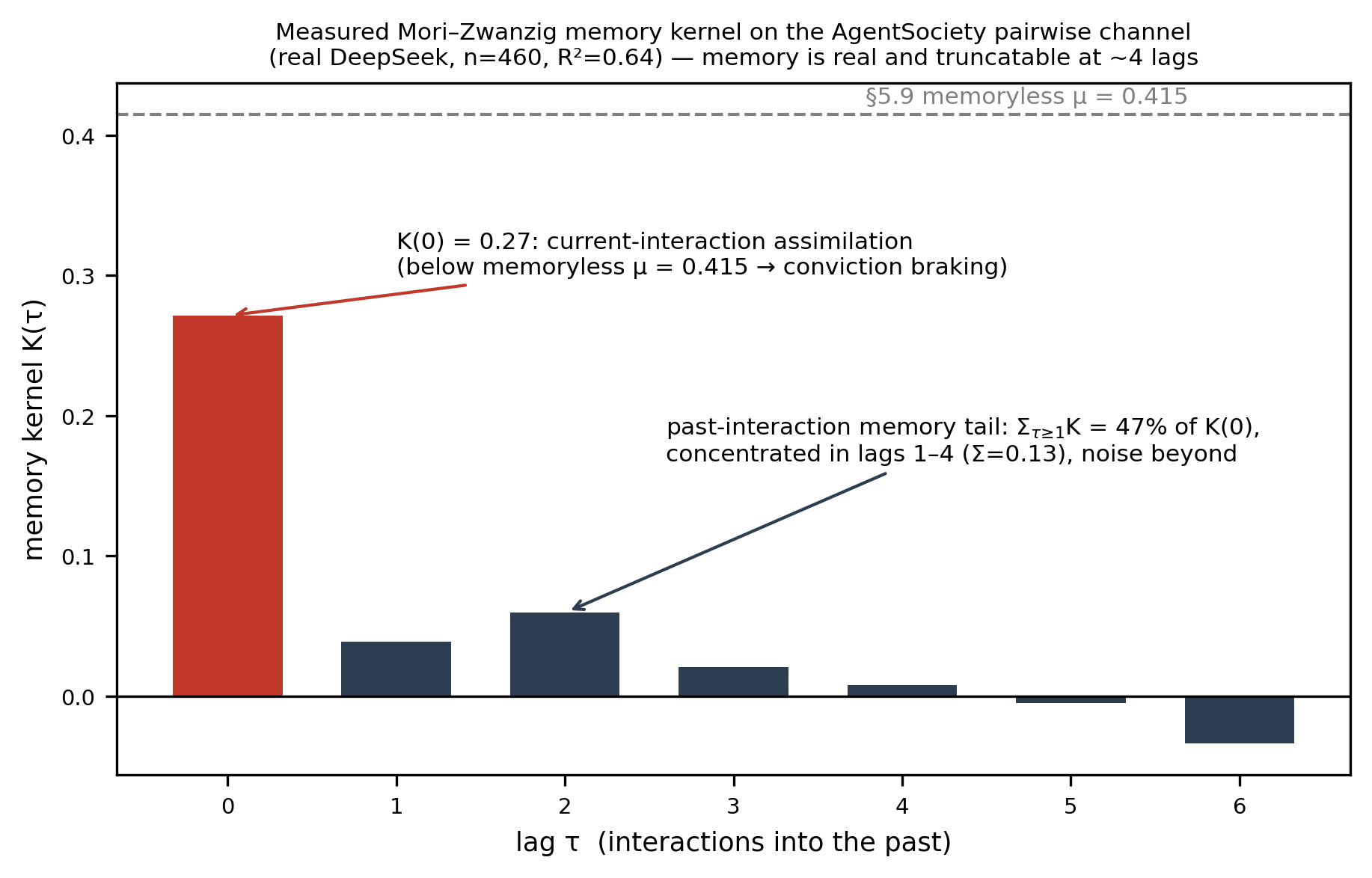}
  \caption{\textbf{A measured memory kernel.} The current assimilation sits below the memoryless
  rate (conviction braking); the past-interaction tail is real, decaying, and truncatable.}
  \label{fig:memory}
\end{figure}

\section{Discussion}\label{sec:discussion}

The results support a single organising claim. Whether a low-parameter surrogate can reproduce the
macroscopic observable of an LLM society is decided by the society's perception and memory design,
and the deciding structure can be read off before the simulation is run.

A global aggregate feed
places the society in a mean-field cell where a scalar surrogate reproduces the observable with an
error that vanishes as $N^{-1/2}$; a community or graph-structured feed places it in cells where the
error is $O(1)$ or grows, and where the surrogate must resolve the responsible structure through a
block, graphon, or pair-plus-memory closure. Two axes beyond interaction order and memory matter in
practice: a shared driver that itself fluctuates, and a curved per-agent response that breaks
coarse-graining through Jensen's inequality even under a private feed. Empirically the operative
control turned out to be the recommender: because it sets a simulation's perception design, it is
what moves a society across the boundary between the regimes where cheap modelling succeeds and where
it fails (Appendix~\ref{sec:targets} and Section~\ref{sec:perception}).

The cheap surrogate is not only a cost-saving device but an instrument. Because its few parameters are
estimated rather than tuned, the value that reproduces a macroscopic law becomes a measurement of
what produces that law: a single labour-cyclicality coefficient for EconAgent's Phillips curve, a
reasoning step rather than phrasing for its magnitude, a saturating societal response for Williams'
epidemic, a decaying memory kernel for AgentSociety's polarisation, and a perception resolution
$\beta_{\rm eff}(N)$ whose crossing of $1$ sets De Marzo's critical group size. Turning the surrogate's fit and
its failures into measurements is, in our view, the more durable contribution.

\section{Limitations}\label{sec:limitations}

Several limitations bound the claims. The named-target reproductions are mechanism-and-scaling
matches, and the two that go further reproduce a target's published \emph{functional form} on a new
model rather than its own model's published number; the strongest test (running a target's exact models to hit its published macroscopic number) remains open. The EconAgent market is a
code-calibrated reimplementation validated against the released non-LLM baseline, and the DeepSeek
decisions drive the agents rather than the market mechanics. The commutation decomposition is an
organising heuristic, not a theorem. The pre-registrations bind specific numeric predictions with
named kill criteria, but each was committed to our own repository together with its result, so the
git history does not by itself separate prediction from data. We document this candidly and archive
the pre-registration bundle, with per-file SHA-256 hashes, at an external timestamped DOI~\cite{prereg};
this anchors the content immutably to a third party, though it dates the deposit, not the prediction. Finally, the LLM-layer results rest on one primary model, and where sibling models were
tested the macroscopic value moved, which is itself part of the finding rather than a nuisance.

\section{Conclusion}\label{sec:conclusion}

We have argued that large LLM societies can, for macroscopic purposes, be modelled without a large
compute budget, and that whether this works is decided by a perception-ordered taxonomy that maps a
simulation's design to an effective theory and a predicted trend of the surrogate error with $N$. We
tested the taxonomy against blind, pre-registered predictions on both of its axes at the LLM layer,
traced its two refuted predictions quantitatively to response curvature, showed the classification to
be automatable, quantified the elicitation and capacity cost of the macroscopic law through a
distillation scaling law, and reproduced eight named LLM simulations and a differentiable agent-based
model on genuine, cheaply elicited decisions.

The recurring lesson is that the surrogate's fit is a
measurement: the microscopic property that a macroscopic observable depends on is exposed, not
hidden, by replacing the expensive agent with a cheap one.

\subsection*{Reproducibility}
Every result has a runner and a cached decision trace, so seeded results reproduce deterministically
and the LLM experiments re-run at no cost; the perception and scaling primitives, the
pre-registrations with their outcomes, and the figure scripts are included. Total elicitation across
the study is a few dollars of DeepSeek.

\subsection*{Declarations}
\noindent\textbf{Competing interests.} The author declares no competing interests.

\smallskip\noindent\textbf{Funding.} This research received no external funding.

\smallskip\noindent\textbf{Data and code availability.} The code, runners, and cached decision traces
that reproduce every result are openly available at
\href{https://github.com/YehudaItkin/poor-mans-agentic-modeling}{github.com/YehudaItkin/poor-mans-agentic-modeling};
the accompanying systematic review and scaling toolkit are archived on Zenodo~\cite{review}.

\clearpage
\appendix

\section*{Appendices}

\section{Proof of Proposition~\ref{prop:floor} (community floor)}\label{app:floor}

\begin{proof}
Write $D_b=A(g^\ast+\delta_b)-A(g^\ast)$ for the deviation of community $b$'s aggregate
response, with $\delta_b$ independent and $\mathcal N(0,\sigma^2)$. The realised population
aggregate under a $B$-community feed is
\[
  \frac1B\sum_b A(g^\ast+\delta_b)=A(g^\ast)+\frac1B\sum_b D_b ,
\]
and the scalar surrogate predicts $A(g^\ast)$, so its error is
$\bigl\lvert\frac1B\sum_b D_b\bigr\rvert$. If $A$ is affine then $D_b$ is Gaussian with mean
$m_1$ and variance $v_1$, and $\frac1B\sum_b D_b\sim\mathcal N(m_1,v_1/B)$ exactly; for smooth
nonlinear $A$ the same holds to leading order, by the central limit theorem in $B$ and a
second-order expansion in $\sigma$. The floor is then the folded-normal mean
\[
  \E\bigl\lvert\mathcal N(m_1,v_1/B)\bigr\rvert
  =\sqrt{\frac{2v_1}{\pi B}}\;e^{-m_1^2 B/(2v_1)}
   +m_1\operatorname{erf}\!\left(m_1\sqrt{\tfrac{B}{2v_1}}\right).
\]
As $B\to\infty$ this tends to $\lvert m_1\rvert$; when $m_1=0$ it equals
$\sqrt{2v_1/(\pi B)}\propto B^{-1/2}$.
\end{proof}

\section{Proof of Proposition~\ref{prop:knee} (the knee $N^\ast$)}\label{app:knee}

\begin{proof}
A private feed assigns each of the $N$ agents an independent draw, i.e.\ $B=N$, so the error is
$\E\lvert\mathcal N(m_1,v_1/N)\rvert$. The folded-normal mean has two regimes,
\[
  \E\bigl\lvert\mathcal N(m_1,v_1/N)\bigr\rvert\approx
  \begin{cases}
    \sqrt{2v_1/(\pi N)}, & \sqrt{v_1/N}\gg\lvert m_1\rvert\quad(\text{decays as }N^{-1/2}),\\[6pt]
    \lvert m_1\rvert,    & \sqrt{v_1/N}\ll\lvert m_1\rvert\quad(\text{constant}).
  \end{cases}
\]
The regimes cross over at $v_1/N=m_1^2$, i.e.\ $N^\ast=v_1/m_1^2$. A second-order expansion of $A$
about $g^\ast$ gives
\[
  m_1=\tfrac12 A''(g^\ast)\sigma^2+O(\sigma^4),\qquad
  v_1=A'(g^\ast)^2\sigma^2+O(\sigma^4),
\]
whence
\[
  N^\ast=\frac{4\,A'(g^\ast)^2}{A''(g^\ast)^2\,\sigma^2}.
\]
\end{proof}

\section{Proof of Proposition~\ref{prop:commute} (exact commutation at the mean-field cell)}\label{app:commute}

\begin{proof}
Averaging the update over the population,
\[
  g_{t+1}=\frac1N\sum_i x_{i,t+1}=a\,g_t+f(g_t)+\frac1N\sum_i\xi_{i,t},
\]
because every agent perceives the same $g_t$, so $f(g_t)$ is a constant pulled out of the sum.
The noise average has mean $0$ and variance $\sigma^2/N$, so $\E[g_{t+1}\mid g_t]=a\,g_t+f(g_t)$
exactly, with an $O(N^{-1/2})$ fluctuation; the response curvature never enters. If instead agent
$i$ perceives $g_t+\varepsilon_i$ with $\varepsilon_i$ mean-zero and symmetric of variance $\tau^2$, the mean
update carries
\[
  \frac1N\sum_i f(g_t+\varepsilon_i)=f(g_t)+\tfrac12 f''(g_t)\tau^2+O(\tau^4),
\]
an $O(1)$ bias that does not average away: the Jensen term of Proposition~\ref{prop:floor}. A
dependence of $x_{i,t+1}$ on the history adds a memory term, and replacing $g_t$ by a
neighbourhood average makes the mean an insufficient statistic, adding the interaction-order term.
\end{proof}

\section{Proof of Proposition~\ref{prop:demarzo} (consensus threshold and its crossing)}\label{app:demarzo}

\begin{proof}
The map $m\mapsto\tanhh(\beta m)$ has derivative $\beta\operatorname{sech}^2(\beta m)$, which at
the disordered fixed point $m=0$ equals $\beta$. The map is odd, so $m=0$ is stable for $\beta<1$
and loses stability in a supercritical pitchfork at $\beta_c=1$. For the amplitude, expand
\[
  \tanhh(\beta m)=\beta m-\tfrac13\beta^3 m^3+O(m^5);
\]
a nonzero self-consistent root of $m=\tanhh(\beta m)$ then satisfies $1=\beta-\tfrac13\beta^3 m^2$,
i.e.
\[
  m^{\ast2}=\frac{3(\beta-1)}{\beta^3},\qquad\text{so}\qquad
  m^\ast\simeq\sqrt{3(\beta-1)}\ \text{ as }\beta\downarrow1.
\]
The existence of the ordered branch is governed by $\beta$ alone and is independent of $N$. A
finite critical group size can therefore arise only if the \emph{measured} slope
$\beta_{\rm eff}(N)$ (the resolution with which an agent reads a weak majority in a list of $N$
opinions) decays through $1$, which is a property of perception, not of the thermodynamic limit.
\end{proof}

\section{External validation suite}\label{app:suite}

The main text carries the load-bearing experiments: EconAgent, the two blind tests, De Marzo, and the measured memory kernel. This appendix collects the remaining validations and supporting checks: further named LLM simulations spanning the other perception cells, a cross-domain comparison against automatic differentiation, a distillation scaling law, an automated cell classifier, and two negative results. Each is a self-contained confirmation of the taxonomy, not a separate study.

\subsection{Named targets across the perception cells}\label{sec:targets}

We now classify and reproduce named LLM simulations, each on genuine LLM decisions (DeepSeek unless noted), using
the recipe: a low-parameter response, the minimal closure named by its cell, and the cell's
predicted $N$-trend. None matches a target's exact same-model published number, which would require
running that target's own models; Section~\ref{sec:demarzo} and Appendix~\ref{sec:williams} come closest by
reproducing a target's published functional form on a new model.

\textbf{AgentTorch: behaviour global, contagion local.}
AgentTorch~\cite{agenttorch} is a million-agent LLM epidemic model (its case study is COVID-19 in New
York City); it reaches that scale by querying the LLM once per demographic archetype and
broadcasting the answer. We confirm from real DeepSeek archetype decisions why this works: the
isolation behaviour is a global feed, so the archetype error shrinks with the number of archetypes,
whereas contagion runs on a contact graph, where a well-mixed surrogate over-predicts the peak and
the break is driven by clustering rather than by the degree tail (Figure~\ref{fig:agenttorch}).

\begin{figure}[t]
  \centering
  \includegraphics[width=0.98\linewidth]{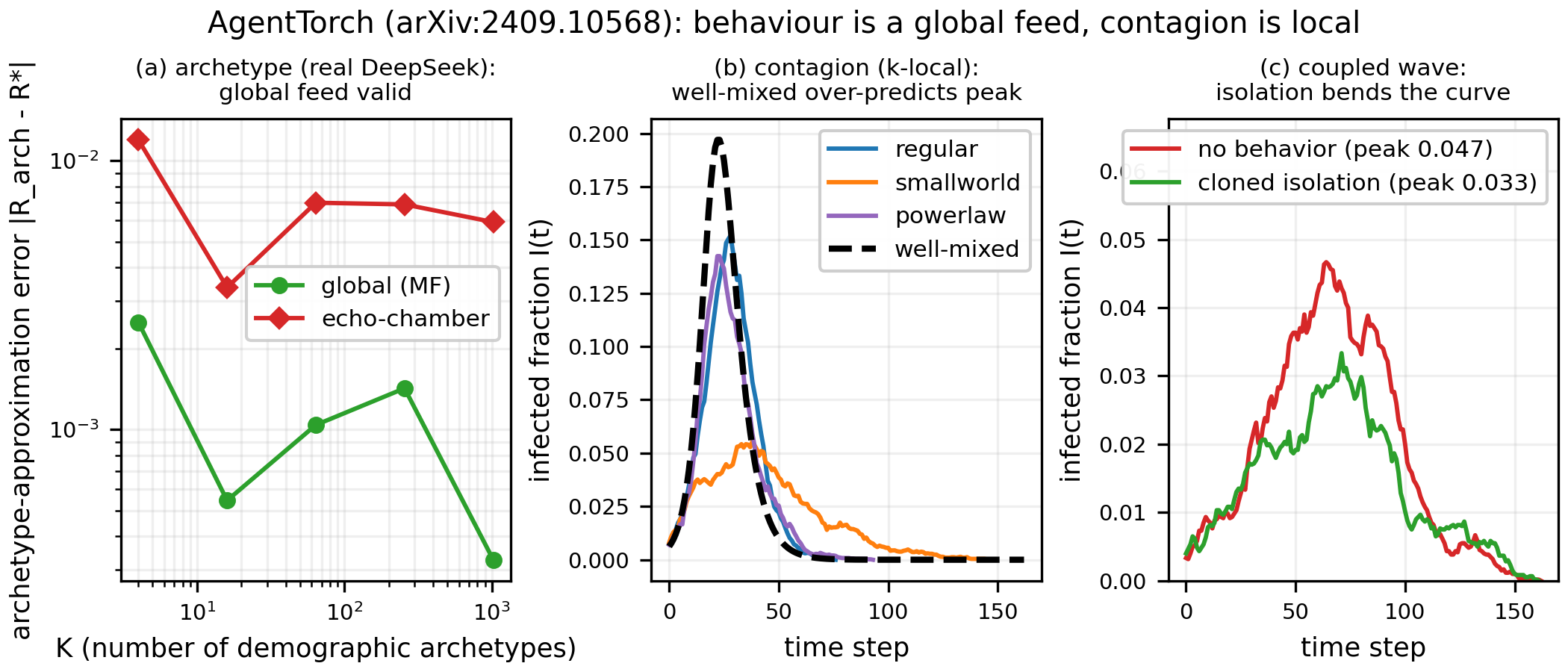}
  \caption{\textbf{AgentTorch spans two cells.} Behaviour is a global feed where the archetype trick
  is valid; contagion is local, where a well-mixed surrogate over-predicts the peak and clustering,
  not the degree tail, drives the break.}
  \label{fig:agenttorch}
\end{figure}

\textbf{OASIS: the recommender is the switch.}
OASIS~\cite{oasis} is a social-media simulation of up to a million agents, modelled on X and Reddit;
it ships two recommender modes, and the mode is the perception switch. Its Reddit hot-score feed
is a single global leaderboard: the cloned herd experiment converges to its mean-field value with an
error that effectively vanishes. Its interest feed is a per-community echo chamber: the
group-polarisation error is $O(1)$ and falls as $B^{-1/2}$ in the number of communities, and a
block-aware surrogate repairs it (Figure~\ref{fig:oasis}). The stance response fitted from $800$ real
DeepSeek decisions weights the agent's own prior ($\alpha=1.17$) above the feed ($\kappa=0.60$), yet
the community feed still breaks the mean field. This response is odd about the operating point, so its
Jensen bias vanishes (Consequence~\ref{cons:floor}) and the break here is a pure interaction-order
effect, not curvature.

\begin{figure}[t]
  \centering
  \includegraphics[width=0.98\linewidth]{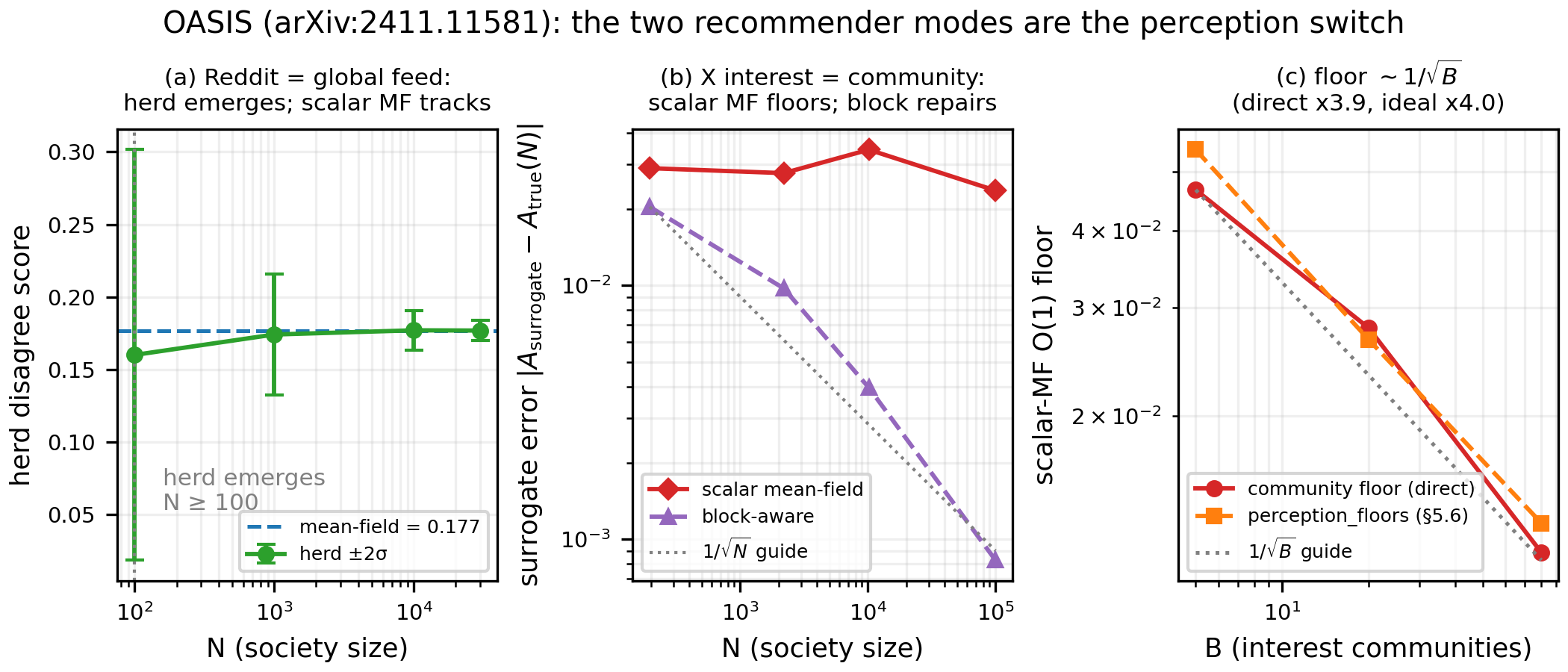}
  \caption{\textbf{OASIS: the two recommender modes are the perception switch.} Global hot-score
  feed $\to$ mean field; per-community interest feed $\to$ an $O(1)$ floor falling as $B^{-1/2}$.}
  \label{fig:oasis}
\end{figure}

\textbf{AgentSociety: the hardest cell.}
AgentSociety~\cite{agentsociety} simulates over ten thousand agents in a data-grounded urban society,
and sits in the local, long-memory cell. A scalar mean field cannot represent its
between-block polarisation at all (the first-moment closure returns zero, so its error is the
polarisation itself), and the error grows with $N$ under a densifying interaction schedule while
remaining a bounded $O(1)$ floor at fixed degree. A pair-plus-memory closure repairs it, and an
ablation of the conviction memory collapses most of the polarisation, isolating the memory axis
(Figure~\ref{fig:agentsociety}). Only the assimilation rate ($\mu=0.415$) is fitted from data here;
Section~\ref{sec:memory} measures the memory dependence that the rest of the mechanism posits.

\begin{figure}[t]
  \centering
  \includegraphics[width=0.98\linewidth]{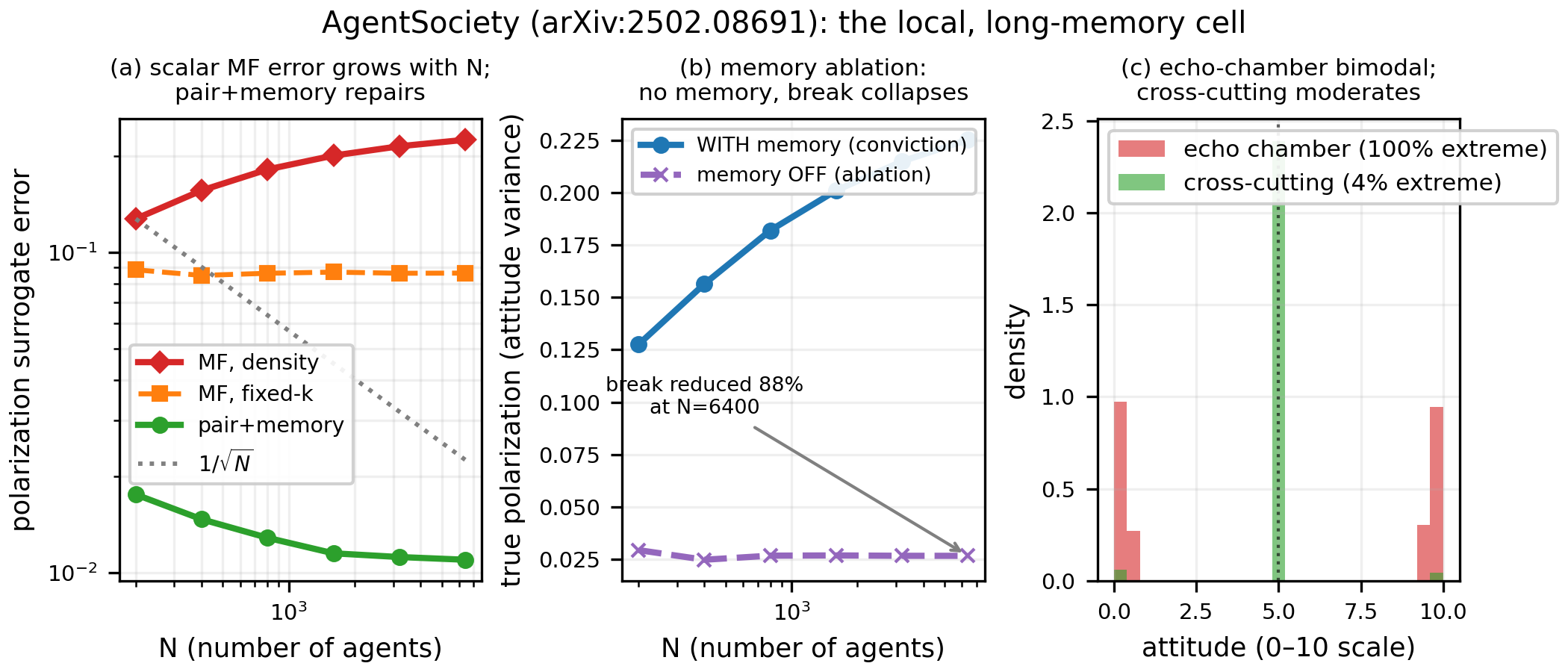}
  \caption{\textbf{AgentSociety in the local, long-memory cell.} Scalar mean-field error grows under
  densification and floors at fixed degree; a pair-plus-memory closure repairs it, and removing the
  conviction memory collapses the break.}
  \label{fig:agentsociety}
\end{figure}

\subsection{A cross-domain check: closure versus autodiff}\label{sec:gradabm}

The recipe is not specific to LLM agents. GradABM~\cite{gradabm} makes a million-agent epidemic differentiable and
calibrates it by gradient descent on GPUs. For the aggregate mortality curve this is more machinery
than the science needs. On a self-generated network epidemic we lift a heterogeneous mean field from
susceptible-infected to the full compartmental model, thirteen degree blocks coupled through a
degree-weighted infectious fraction, and calibrate two parameters by a derivative-free simplex. It
recovers the transmission rate to within $13\%$ and the fatality rate to within $6\%$ at a
mortality-curve RMSE of $22$, in $0.34$ seconds on a laptop, against a well-mixed baseline that
misfits at RMSE $48$ (a factor of $2.2$) and against roughly $400$ seconds of GPU
forward-time per county for the differentiable model (Figure~\ref{fig:gradabm}). We recover the
method and the planted parameters, not the authors' number on their data.

\begin{figure}[t]
  \centering
  \includegraphics[width=0.98\linewidth]{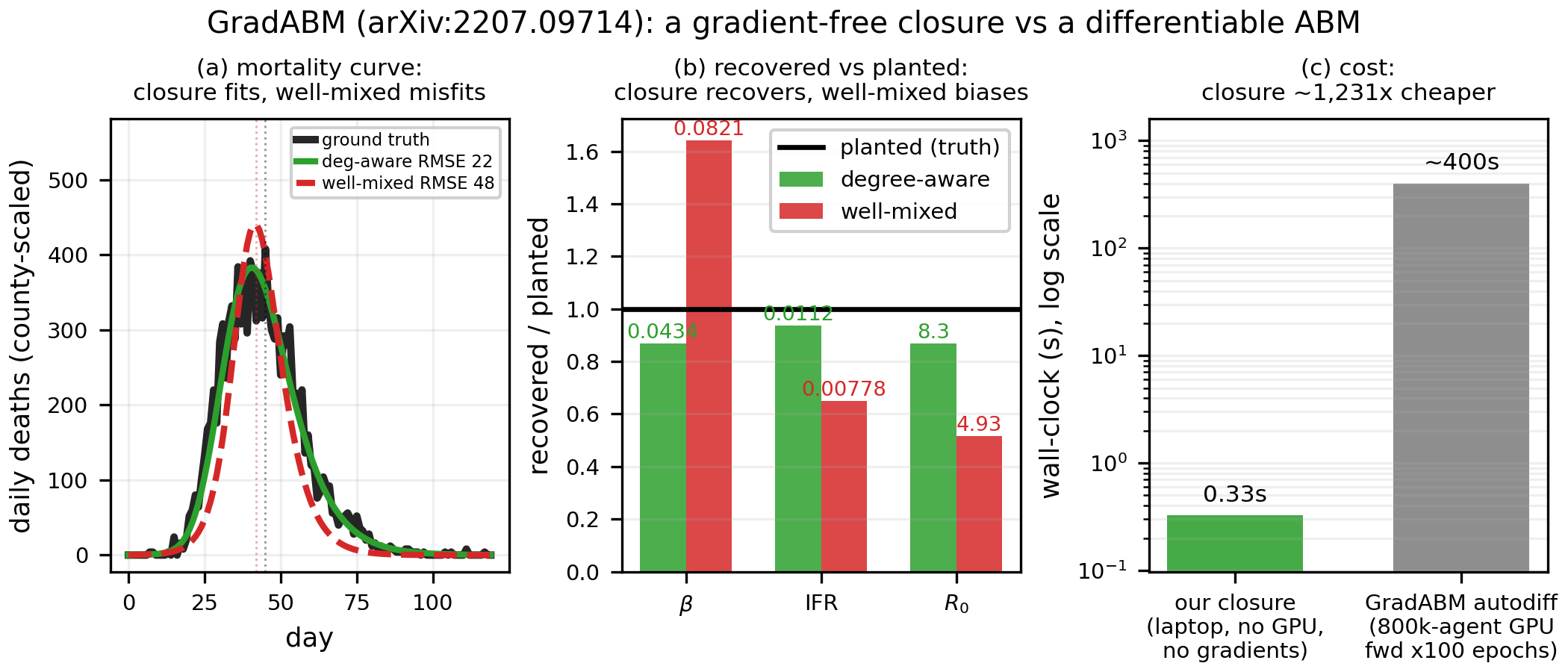}
  \caption{\textbf{Closure versus autodiff.} A gradient-free degree-aware closure recovers the
  mortality curve and the planted parameters roughly $10^3$ times faster than the differentiable
  million-agent model, for the aggregate observable.}
  \label{fig:gradabm}
\end{figure}

\subsection{A distillation scaling law}\label{sec:scaling}

How much LLM does the macro law cost? Offline from the cached EconAgent trace we measure the Phillips
error as a joint function of the elicitation budget $B$, the surrogate capacity $p$, and the
population $N$ (Figure~\ref{fig:scaling}). The error falls with budget, reaching tolerance by a
couple of thousand decisions; it falls with capacity and plateaus at four features; and the two axes
are approximately separable, the error grid being ninety percent rank-one in log-error. The
load-bearing feature is a capacity floor: below four features the error plateaus above tolerance for
any budget, so the surrogate is capacity-limited and buying more decisions cannot help, whereas above
it the surrogate is data-limited and converges. The purchasing rule is therefore to buy structure
first and data second: four features and one to two thousand decisions, well under a dollar of DeepSeek.

\begin{figure}[t]
  \centering
  \includegraphics[width=0.98\linewidth]{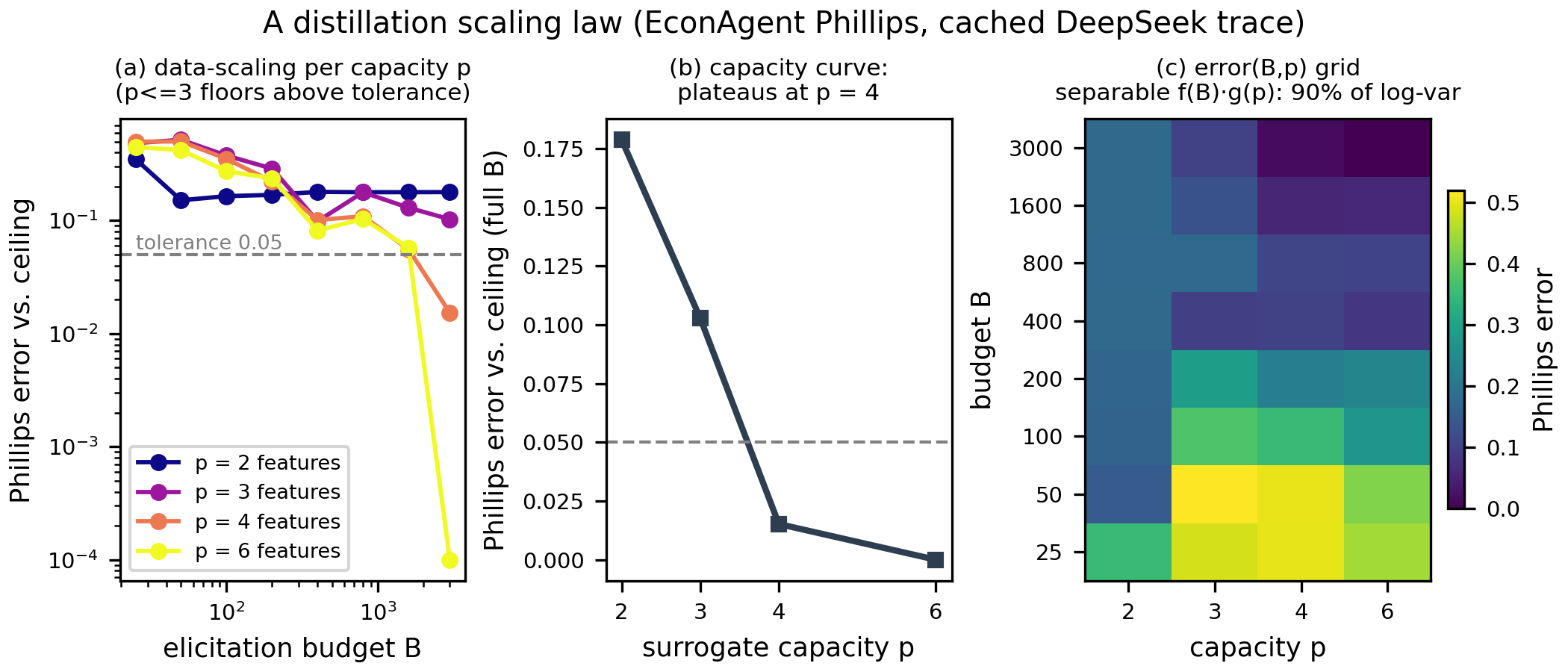}
  \caption{\textbf{A distillation scaling law.} The macro-observable error against elicitation
  budget $B$ and surrogate capacity $p$; a capacity floor no budget can cross, and an approximately
  separable grid.}
  \label{fig:scaling}
\end{figure}

\subsection{The classification is automatable}\label{sec:config2cell}

The cell assignments above were made by hand. They need not be. Given only a neutral description of a
simulation's perception and memory design, with no taxonomy term leaked, an LLM emits the structured
specification that \texttt{classify\_cell} consumes; the resulting cell matched the hand assignment
in all eight cases we tested, including Generative Agents, correctly placed in the local,
long-memory cell; the classifier saw no taxonomy term, though we wrote the neutral descriptions and
cannot rule out that these published systems appear in its training data. The classificatory layer is a pre-flight
check for which surrogate will work, runnable before any expensive simulation.

\subsection{Williams et al.: a generative epidemic in the mean-field cell}\label{sec:williams}

Williams et al.~\cite{williams} give generative agents a daily stay-home decision from their symptoms
and a town-wide broadcast of new-infection prevalence. They publish three signatures: a stay-home
logistic regression (positive on symptoms, positive on prevalence, negative on prevalence squared,
i.e.\ a saturating response); a flattened epidemic curve; and a shift from a single wave to several
as the reproduction number falls. From $462$ real DeepSeek decisions on their verbatim prompt we
reproduce all three levels (Figure~\ref{fig:williams}). The fitted logistic has the published sign
structure, including the negative squared term ($-1.77$, a stronger saturation than their $-0.65$),
tying the epidemic's societal channel to the same response curvature that
Section~\ref{sec:perception} isolated. Dropping the fitted response into a behaviour-coupled
compartmental model flattens the peak to a seventh of the no-feedback baseline and roughly doubles
the duration, with a single wave at high reproduction number and multiple waves at low. And the cell
is mean field: the scalar surrogate's error against a finite-$N$ stochastic model shrinks with $N$.

\begin{figure}[t]
  \centering
  \includegraphics[width=0.98\linewidth]{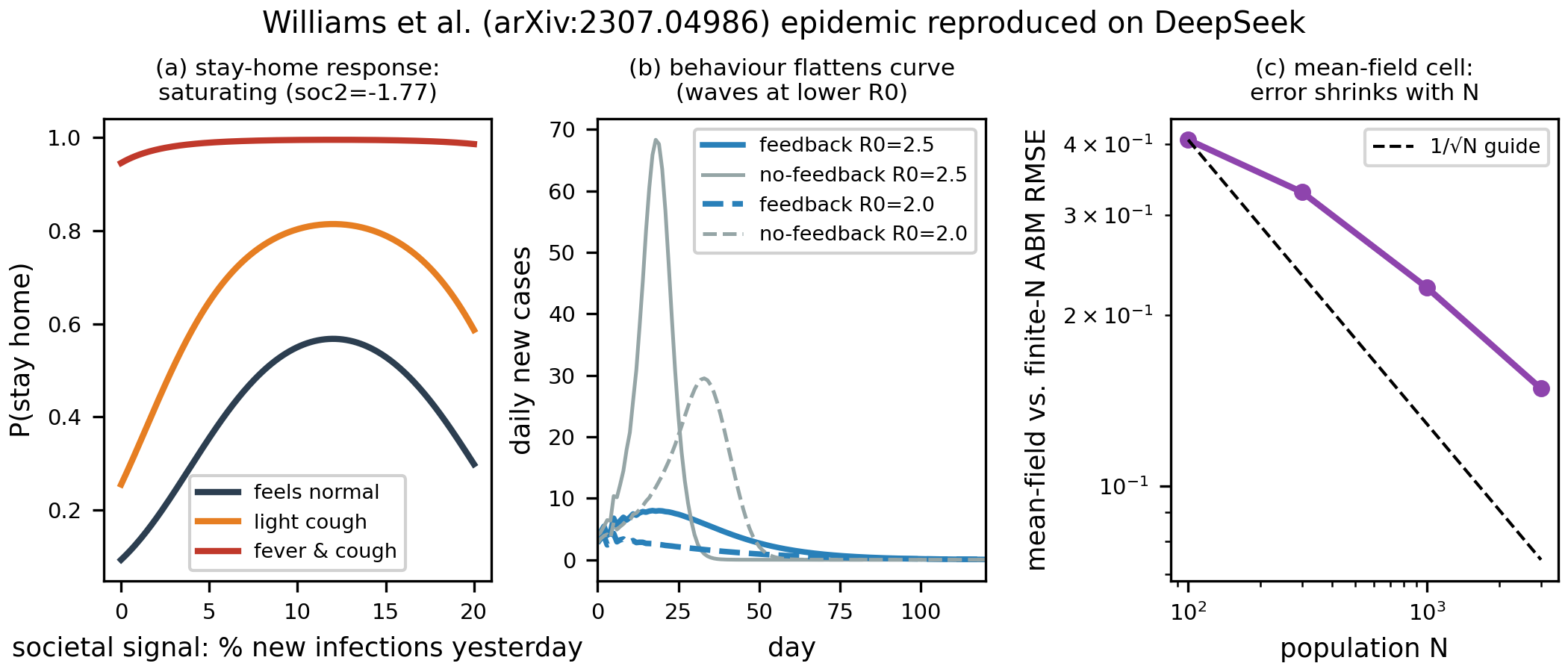}
  \caption{\textbf{Williams et al.\ epidemic on DeepSeek.} A saturating stay-home response
  reproduces the published form; the behaviour-coupled model flattens the curve and shows waves; the
  surrogate error shrinks with $N$.}
  \label{fig:williams}
\end{figure}

\subsection{LLMTraveler: filling the empty cell}\label{sec:traveler}

Every target so far has short memory except AgentSociety, which is $k$-local. The
[mean-field $\times$ long-memory] cell (a global signal with an accumulated history) was empty.
LLMTraveler~\cite{traveler} fills it: LLM commuters play a day-to-day congestion game, each day
seeing the exponentially weighted moving average of their travel time on each route (the memory)
and choosing a route, with congestion felt only through the aggregate flow.

The published result is
that mean travel times converge to the Dynamic User Equilibrium (DUE), most models landing within
$\pm10\%$ (GPT-4o $+1.33\%$). We elicit real DeepSeek route choices and fit a two-parameter rule
$P(\text{switch})=\sigma(\beta\,\Delta-\gamma)$ in the travel-time gap $\Delta$
(prereg~\texttt{prereg\_llmtraveler\_2026-07-03.md}, Figure~\ref{fig:traveler}). The rule is rational
($\beta=+0.40$, it prefers the faster route) though it shows little of the inertia the authors
report ($\gamma\approx0$ on DeepSeek, a model-dependent difference). Run as a day-to-day dynamic on
the 16-traveler two-route network, it \emph{converges to the DUE at $+4.7\%$}, inside the published
$\pm10\%$ band, with persistent switching around equilibrium as the authors also observe. The cell
is confirmed too: the equilibrium gap $\lvert\text{mean}-\text{DUE}\rvert$ shrinks with the number
of travelers (a mean-field flow convergence), so the surrogate reproduces the Wardrop user
equilibrium~\cite{wardrop} (the flow at which no traveler can reduce their own travel time by switching route
unilaterally, here the same fixed point as the DUE above) at any $N$ on a laptop.

\begin{figure}[t]
  \centering
  \includegraphics[width=0.98\linewidth]{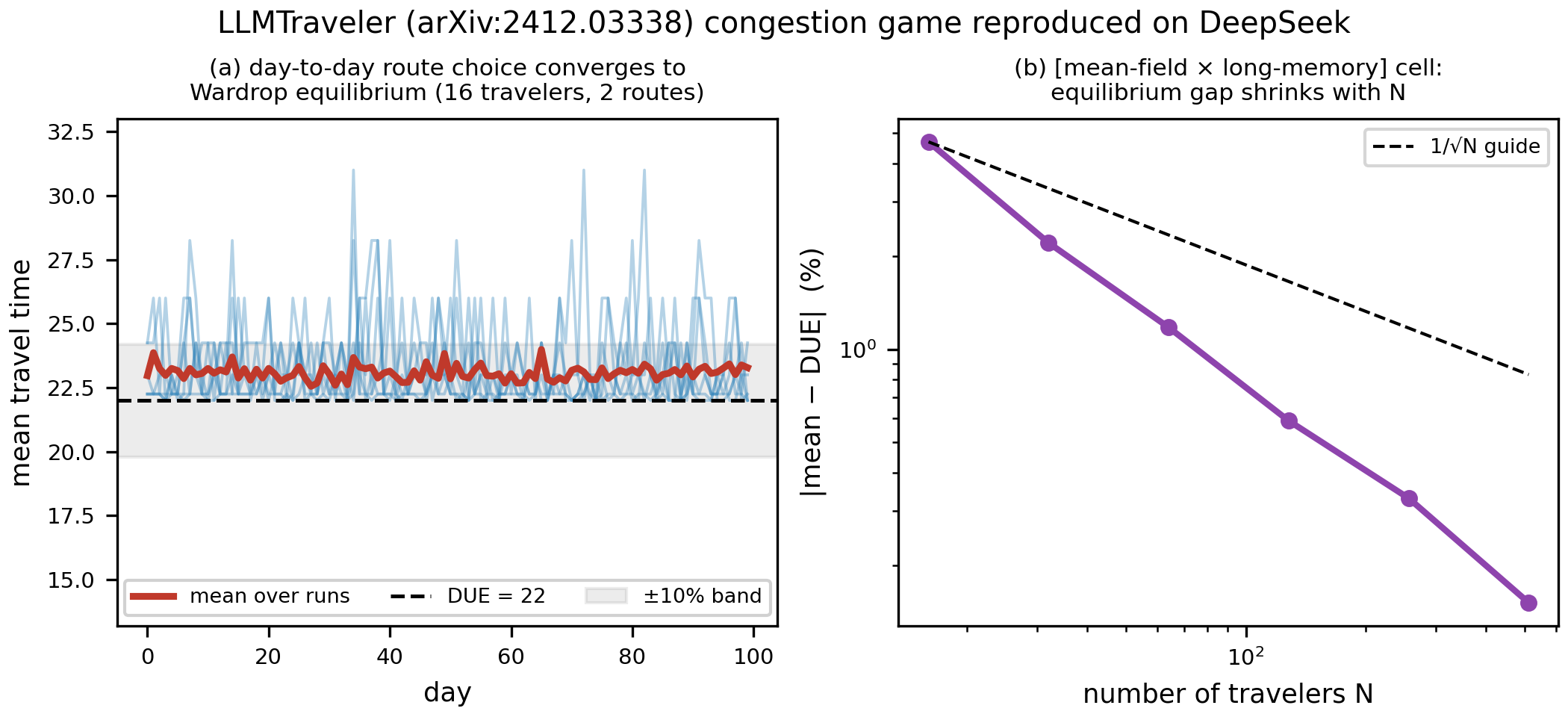}
  \caption{\textbf{LLMTraveler in the mean-field $\times$ long-memory cell.} A two-parameter route
  rule fitted to real LLM choices converges to the Dynamic User Equilibrium within the published
  $\pm10\%$ band (left); the equilibrium gap shrinks with the number of travelers (right).}
  \label{fig:traveler}
\end{figure}

\subsection{TwinMarket: the stylised facts need the market, not just the agent}\label{sec:twinmarket}

Not every target yields to a compact surrogate, and the failure is informative. TwinMarket~\cite{twinmarket} reports that GPT-4o BDI investors reproduce the canonical stylised facts of
financial returns (fat tails and volatility clustering) from a rich belief--desire--intention
state with a social feed. We elicit real trading decisions (DeepSeek and GPT-4o) on a compact state
(recent trend, peer signal, own P\&L) and fit a low-parameter trader
(prereg~\texttt{prereg\_twinmarket\_2026-07-03.md}). The trader recovers the behavioural \emph{signs}
(trend-following, herding on the peer signal, a disposition effect, selling into profit) on both models.

Yet dropped into a minimal price-impact market at a moderate coupling the trader does
\emph{not} generate fat tails: the returns stay near-Gaussian (excess kurtosis $\approx0$), because
the elicited herding leaves the market subcritical. The boundary here is not the trader's capacity
but an \emph{un-elicited} ingredient: the strength of the price-impact coupling. Raising that
coupling makes the \emph{same} weak trader supercritical (excess kurtosis rises past $7$), so the
stylised facts are under-determined by the elicited response alone: reproducing them needs the market
mechanism, not just the agent, and our compact surrogate does not fit that mechanism from LLM
decisions. We report this honestly as the limit of eliciting only the agent, distinct from the
data/capacity floor of Appendix~\ref{sec:scaling}.

\subsection{Smallville: the field's most-cited numbers}\label{sec:smallville}

Generative Agents~\cite{genagents} produced the most-cited macro numbers in the area: in Smallville's
25 agents, an invitation to a Valentine's party spreads by word-of-mouth to $13$ of $25$, and $5$
attend. This is the [$k$-local $\times$ long-memory] cell: a local acquaintance cascade with agents
remembering they were told. We elicit real DeepSeek transmission and attendance decisions, fit two
logits ($P(\text{tell})$ rising with tie strength; $P(\text{attend})$), and run the two-day cascade
on a $25$-node acquaintance graph (prereg~\texttt{prereg\_smallville\_2026-07-03.md}). The
reproduction is close but semi-quantitative rather than parameter-free: at a plausible acquaintance
degree ($\sim6$) the cascade reaches $\mathbf{11.4}$ of $25$ (published $13$) and $\mathbf{4.5}$
attend (published $5$). Because we lack Smallville's actual relationship graph, the reach depends on
the assumed network (from $\sim7$ at degree $4$ to $\sim19$ at degree $10$), so a realistic social
degree ($6$--$8$) brackets the published $13$ but we do not claim the exact number.

What \emph{is}
robust is the cell: a well-mixed control that lets every informed agent talk to everyone reaches all
$25$, overshooting, whereas the local cascade does not: the diffusion is a property of the
\emph{local} interaction structure, exactly the $k$-local prediction.

\end{document}